\documentclass[twoside,11pt]{article}
\usepackage{jmlr2e}

\usepackage{blindtext}

\usepackage{changes}

\usepackage{color}
\usepackage{url}
    
\usepackage{amssymb,amsmath}

\usepackage[mathcal]{eucal}

\DeclareMathAlphabet\mathbfcal{OMS}{cmsy}{b}{n}

\newcommand{\BEAS}{\begin{eqnarray*}}
\newcommand{\EEAS}{\end{eqnarray*}}
\newcommand{\BEA}{\begin{eqnarray}}
\newcommand{\EEA}{\end{eqnarray}}
\newcommand{\BEQ}{\begin{equation}}
\newcommand{\EEQ}{\end{equation}}
\newcommand{\BIT}{\begin{itemize}}
\newcommand{\EIT}{\end{itemize}}
\newcommand{\BNUM}{\begin{enumerate}}
\newcommand{\ENUM}{\end{enumerate}}
\newcommand{\BA}{\begin{array}}
\newcommand{\EA}{\end{array}}

\newcommand{\Diag}{\mathop{\rm Diag}}

\newcommand{\tr}{\mathop{ \rm tr}}

\newcommand{\idm}{I}
\newcommand{\rb}{\mathbb{R}}
\newcommand{\cb}{\mathbb{C}}

\newcommand{\mysec}[1]{Section~\ref{sec:#1}}
\newcommand{\eq}[1]{Eq.~(\ref{eq:#1})}
\newcommand{\myfig}[1]{Figure~\ref{fig:#1}}

\def \E{{\mathbb E}}

\def \H{{\mathcal H}}

\def \ds{\displaystyle}

\def \sti{Stieltjes }
\title{On the Effectiveness of the $z$-Transform Method \\ in Quadratic Optimization}

\author{\name Francis Bach \email francis.bach@inria.fr \\
\addr Inria - Ecole Normale Sup\'erieure\\ 
PSL Research University}

\date{\today}

\editor{Raef Bassily}

\usepackage{lastpage}
\jmlrheading{27}{2026}{1-\pageref{LastPage}}{7/25; Revised
8/26}{8/26}{25-1621}{Francis Bach}
\ShortHeadings{On the Effectiveness of the $z$-Transform Method in Quadratic Optimization}{Bach}

\begin{document}

\maketitle

\begin{abstract}
The $z$-transform of a sequence is a classical tool used in signal processing, control theory, computer science, and electrical engineering. It allows one to study sequences from their generating functions, with many operations that can be equivalently defined on the original sequence and its $z$-transform. In particular, the $z$-transform method focuses on asymptotic behaviors and allows the use of Taylor expansions. We present a sequence of results of increasing significance and difficulty for linear models and optimization algorithms, demonstrating the effectiveness and versatility of the $z$-transform method in deriving new asymptotic results. Starting from the simplest gradient descent iterations in an infinite-dimensional Hilbert space, we show how the spectral dimension characterizes the convergence behavior. We then extend the analysis to Nesterov acceleration, averaging techniques, and stochastic gradient descent.
\end{abstract}

\section{Introduction}
Characterizing the convergence of real-valued or vector-valued sequences is a key theoretical problem in data science, where the sequence index typically corresponds to the number of iterations of an iterative algorithm (such as in optimization and signal processing) or the number of observations (as in statistics and machine learning). This characterization can be done mainly in two ways, asymptotically or non-asymptotically. In an asymptotic analysis, an asymptotic equivalent of the sequence is identified, which readily allows comparisons with other algorithms; however, without further analysis, the behavior at any finite time cannot be controlled. This is exactly what non-asymptotic analysis aims to achieve, by providing bounds that are valid even for a finite index, but then only providing bounds that cannot always be compared. While the two approaches have their own merits, in this paper, we focus on asymptotic analysis and sequences that tend to their limit at a sub-exponential rate that is a power of the sequence index.

The main goal of this paper is to show how a classical tool from signal processing, control theory, and electrical engineering~\citep{oppenheim1997signals}, the $z$-transform method~\citep{jury1964theory}, can be used in this context with striking efficiency in obtaining asymptotic equivalents for the class of algorithms that can be seen as iterations of potentially random linear operators in a Hilbert space. This includes gradient descent for quadratic optimization problems as well as its accelerated and stochastic variants~\citep{nesterov2018lectures}, Landweber iterations in inverse problems~\citep{benning2018modern,hanke1991accelerated}, or gossip algorithms in distributed computing~\citep{boyd2006randomized}. {Since we consider iterations of linear operators, we will obtain performance criteria, denoted $a_k$ below after $k$ iterations, that are linear combinations of powers of the form $\lambda^k$, with $\lambda$ potentially arbitrarily close to $1$. This leads to convergence behavior of various types (mostly in the form $k^{-\omega}$ for $\omega>0$).}

In a nutshell, the asymptotic behavior of the real-valued sequence $(a_k)_{k \geqslant 0}$ when $k$ tends to $+\infty$ can be characterized by the asymptotic behavior of the $z$-transform (a.k.a.~generating function)
$$A(z) = \sum_{k=0}^{+\infty} a_k z^k$$
 when $z$ tends to $1$ from below, with formal equivalences of the form $$\lim_{k\to +\infty} a_k = \lim_{z \to 1^-}\  (1-z) A(z),$$
 a result which is often referred to as the ``final value theorem.'' What makes the $z$-transform method particularly versatile in the analysis of sequences is {(1) the possibility of characterizing the asymptotic behavior $a_k$ beyond its limit (such as convergence in $1/k^\omega$) using the fine asymptotic behavior of $A$ and its derivatives around $1$, and (2)} the translation of common transformations of discrete-time sequences into simple transformations of the continuous-parameter $z$-transform (such as differentiation or multiplication). Thus, the $z$-transform allows one to benefit from differential calculus without resorting to limiting continuous-time approximations, which are common in optimization~\citep{su2016differential,wibisono2016variational,scieur2017integration}. 
 
In data science, numerical sequences are often obtained as squared norms of vector-valued sequences (e.g., the distance to the optimum in optimization). We will, therefore, need to consider the squared norms of sequences in Hilbert spaces, where the associated compact linear operator will have a discrete but infinite spectrum, i.e., a sequence of eigenvalues that go to zero; this will lead to a specific spectral measure, which is accessed through its \sti transform~\citep{widder1942laplace}. A key quantity is then the spectral dimension, as outlined by~\citet{berthier2020accelerated} in the context of gossip algorithms \citep{boyd2006randomized}, which characterizes the rate of convergence of the sequence for a wide variety of situations.
 
\paragraph{Contributions and paper outline.}
We present a sequence of results of increasing significance, novelty, and difficulty for linear models and optimization algorithms, showing the effectiveness of the $z$-transform in deriving existing and new asymptotic results. These results hold asymptotically as the number of iterations tends to infinity, but provide more precise scaling laws compared to worst-case bounds.
\BIT
\item In \mysec{z}, we review the classical properties of $z$-transforms of discrete sequences and the conditions under which final value theorems are valid, so-called Tauberian conditions~\citep{korevaar2004tauberian}.

\item  In \mysec{GD}, we study the simplest linear recursion that corresponds to gradient descent for a quadratic optimization problem, highlighting that, following \citet{berthier2020accelerated}, the spectral dimension (denoted $\omega$) characterizes the convergence behavior, through a specific \sti transform, with a rate proportional to $ {1}/{k^\omega}$ after $k$ iterations.
\item  In \mysec{nesterov-1}, we show how Nesterov acceleration~\citep{nesterov1983method,nesterov2018lectures} with a particular momentum term proposed by \citet{flammarion2015averaging}  can improve the rate for gradient descent by a simple extrapolation technique, leading to an improved rate proportional to $ {1}/{k^{\min\{ 2\omega, \omega + 1\}}}$, which provides scaling laws complementing the upper bounds derived by \citet{flammarion2015averaging}.  We show in \mysec{HB} that the related heavy-ball acceleration technique is a counterexample for the $z$-transform method, as oscillating behaviors prevent simple asymptotic equivalents. The study of acceleration is completed in \mysec{nesterov-2}, where we extend the acceleration technique developed in \mysec{nesterov-1} with a more general extrapolation step, parameterized by some integer $\rho$, where the convergence rate is proportional to $ {1}/{k^{\min\{ 2\omega, \omega + \rho\}}}$, with explicit constants, and only a partial proof.

\item    In \mysec{additivenoise}, we show how the $z$-transform method can be easily used to handle additive homoscedastic noise in linear iterations, while in \mysec{averaging}, we study averaging techniques, where again the $z$-transform method can be applied efficiently to derive asymptotic results.

\item   In \mysec{sgd}, we review how stochastic gradient descent for least-squares regression can be analyzed with a recursion on the covariance matrix. We show that, under extra assumptions on the data that include Gaussian inputs, this leads to simple asymptotic results in expectation, with a similar convergence rate proportional to $1/k^\omega$, but with a smaller stepsize than the deterministic counterpart. This recovers results from~\citet{velikanov} with a more direct argument. The $z$-transform has already been used in the context of stochastic gradient descent for stability analysis~\citep{feuer2003convergence,horowitz1981performance} and for deriving convergence rates with constant momentum~\citep{velikanov2023a} (we explore the use of $z$-transform in a more systematic way, in particular for time-varying recursions and directly for performance measures rather than their cumulative sums).

\EIT

\section{Review of the $z$-transform method}
\label{sec:z}
Given a real-valued sequence $(b_k)_{k \geqslant 0}$, we define its $z$-transform as $$
 B(z) =  \sum_{k=0}^{+\infty} b_k z^k.$$
 Note the non-standard convention of using $z^k$ instead of $z^{-k}$, chosen to preserve the link with power series and generating functions~\citep{wilf2005generatingfunctionology}. In this paper, we will use $z \in \cb$ when relating to Fourier series and when taking derivatives, but all asymptotic equivalents will be taken for $z \in \rb$. Throughout this paper, we will assume that the series is absolutely convergent for $|z|<1$ (a sufficient condition is that $k^\alpha|b_k|$ is bounded for some $\alpha \in \rb$). The $z$-transform is a standard tool in applied mathematics~\citep[see, e.g.][]{jury1964theory,kelley2001difference}. 

\subsection{Sufficient conditions for the final value theorem}
\label{sec:tauber}

The $z$-transform method and the final value theorem are intimately related to Abel's theorem, which we state below. See~\citet{korevaar2004tauberian} for more details.

\begin{theorem}[Abel]
If $(b_k)_{k \geqslant 0}$ is a real-valued sequence such that the series defined as $  B(z) =  \sum_{k=0}^{+\infty} b_k z^k$ is convergent on the interval $[0,1)$, then, if $  \sum_{k = 0 }^{+\infty} b_k$ is convergent, we have $  \lim_{z \to 1^-} B(z) = \sum_{k = 0 }^{+\infty} b_k$.
\end{theorem}

This shows that under certain conditions (convergence of the series), the limiting behavior of $(b_k)_{k \geqslant 0}$ when~$k$ tends to infinity is linked to the behavior of $B(z)$ for $z$ tending to $1^-$. There are several converses, with additional conditions (since oscillating sequences such as $b_k = (-1)^k$ provide a classical counterexample), which are often referred to as Tauberian theorems, including the following classical one from Hardy and Littlewood~\citep{hardy1920abel}.

\begin{theorem} \textbf{\emph{\citep[Theorem 7.2, Part I]{korevaar2004tauberian}}}
If $(b_k)_{k \geqslant 0}$ is a real-valued sequence such that the series $  B(z) =  \sum_{k=0}^{+\infty} b_k z^k$ is defined on the interval $[0,1)$, then, if there exists a constant $c \in \rb$ such that $\forall k \geqslant 0$, $ k b_k \leqslant c$, and $  \lim_{z \to 1^-} B(z)$ exists, then the series $  \sum_{k = 0 }^{+\infty} b_k$ is convergent and  $ \sum_{k = 0 }^{+\infty} b_k
=  \lim_{z \to 1^-} B(z) $.
\end{theorem}
Note that the extra condition   $\forall k \geqslant 0$, $ k b_k \leqslant c$ can be replaced by  $\forall k \geqslant 0$, $ k b_k \geqslant -c$. In particular, $b_k \geqslant 0$ is a sufficient condition.

By   considering $b_k = a_k - a_{k-1}$ for $k \geqslant 1$, and $b_0 = a_0$, then with $A(z) = \sum_{k=0}^{+\infty} a_k z^k$, we have $B(z) = (1-z) A(z)$ for all $z \in [0,1)$, and the equality of the two limits leads to the ``final value theorem:''
$$
\ds \lim_{z \to 1^-}\  (1-z) A (z) = \lim_{k \to + \infty} a_k,
$$
where the existence of the right limit implies the existence of the left one (by Abel's theorem), and the opposite direction if we have in addition $\forall k \geqslant 0, \ k (a_k - a_{k-1} )  \leqslant  c$ for some constant $c \in \rb$.

The two theorems can in fact be strengthened significantly to account for more general behaviors where powers need not be integers, and we provide here the equivalent versions with the sequence $(a_k)_{k \geqslant 0}$ rather than $(b_k)_{k \geqslant 0}$\footnote{{For $\alpha \geqslant 1$, Theorem~\ref{theorem:gamma} exactly follows from Theorem I.7.4 from~\citet{korevaar2004tauberian}, applied to $b_k = a_k - a_{k-1}$, while Theorem~\ref{theorem:gamma0} follows from the remark immediately after. For $0<\alpha<1$, it follows from the negative-index companion immediately following Theorem I.7.4, after subtracting the canonical sequence $\frac{\Gamma(k+\alpha)}{\Gamma(\alpha)\Gamma(k+1)}$.}}.
\begin{theorem}{\emph{\textbf{(Adapted from \citet[Theorem I.7.4 and following remarks]{korevaar2004tauberian})}}}
\label{theorem:gamma0} Let $\alpha > 0$. 
If $(a_k)_{k \geqslant 0}$ is a real-valued sequence such that the series $ A(z) =  \sum_{k=0}^{+\infty} a_k z^k$ is defined on the interval $[0,1)$, then, if $\lim_{k \to +\infty} k^{1-\alpha}  a_k$ exists, we have $$  \lim_{z \to 1^-} \ (1-z)^\alpha A(z) = \Gamma(\alpha) \cdot \lim_{k \to +\infty} k^{1-\alpha}  a_k,$$
where $\Gamma$ is the Gamma function.
\end{theorem}

\begin{theorem} {\emph{\textbf{(Adapted from \citet[Theorem I.7.4 and following remarks]{korevaar2004tauberian})}}}
\label{theorem:gamma}
Let $\alpha > 0$. 
If $(a_k)_{k \geqslant 0}$ is a real-valued sequence such that the series $  A(z) =  \sum_{k=0}^{+\infty} a_k z^k$ is defined on the interval $[0,1)$, then, if there exists a constant $c \in \rb$ such that $\forall k \geqslant 1$, $ k^{2-\alpha} (a_k-a_{k-1}) \leqslant c$, and $  \lim_{z \to 1^-} (1-z)^\alpha A(z)$ exists, then $$
\lim_{k \to +\infty} k^{1-\alpha}  a_k = \frac{1}{\Gamma(\alpha) }  \cdot \lim_{z \to 1^-} \ (1-z)^\alpha A(z).$$
\end{theorem}

The last two theorems show that the behavior of $A(z)$ \emph{when it explodes} at $1^-$ as $(1-z)^{-\alpha}$, for $\alpha >0$, gives an asymptotic equivalent for the sequence $(a_k)_{k \geqslant 0}$. A classical example is $a_k = \frac{1}{k!} \alpha(\alpha+1) \cdots (\alpha+k-1) = \frac{ \Gamma(k+\alpha)}{\Gamma(\alpha)\Gamma(k+1)} \sim \frac{k^{\alpha-1}}{\Gamma(\alpha)}$, for which $A(z) = (1-z)^{-\alpha}$.

In our linear-recursion setup, most of the sequences will satisfy the equality of Abel's summability (existence of a limit of $B(z)$ when $z \to 1^-$) and classical summability, but this will not apply to sequences with non-vanishing oscillations such as those obtained from the heavy-ball methods in \mysec{hb} (Nesterov acceleration will lead to \emph{vanishing} oscillations in \mysec{nesterov-1} and \mysec{nesterov-2}).

\paragraph{Derivatives of the $z$-transform.} The results above are limited to speeds of convergence of $(a_k)_{k \geqslant 0}$ strictly slower than $1/k$. To go further, we need to consider the non-exploding behavior of $A(z)$ around $z=1$, for which the situation is a bit more complicated. Indeed, considering $A(z) = (1-z)^\mu$ for an integer $\mu$, we see that the sequence here is equal to zero for $k>\mu$ (and thus with no monomial equivalent). Moreover, $a_k = 1/k$ for $k>1$, leads to $A(z) = - \log(1-z)$, which does not have a behavior in a power of $1-z$. 

To get equivalents, we can apply the theorems above to derivatives of $A(z)$, noticing that for $\mu \in \mathbb{N}$,
\BEAS
A^{(\mu)}(z) & = & \sum_{k = \mu}^{+\infty} k(k-1) \cdots(k-\mu+1) a_k z^{k - \mu}
 \!=\! \sum_{k=0}^{+\infty} (k+\mu)(k+\mu -1) \cdots(k+1) a_{k+\mu} z^k 
 \\
 & = &  \sum_{k=0}^{+\infty} \frac{\Gamma(k+1+\mu)}{\Gamma(k+1)}  a_{k+\mu} z^k,
 \EEAS
which implies that the $\mu$-th order derivative $A^{(\mu)}(z) $ corresponds to the sequence 
$((k+1) \cdots (k+\mu) a_{k+\mu} )_{k \geqslant 0}$ which is asymptotically equivalent to $(k^\mu a_k)_{k \geqslant 0}$. Thus, if the two limits exist, we get for any $\alpha>0$:
\BEQ
\label{eq:derivative} \lim_{z \to 1^-} \ (1-z)^{\alpha} A^{(\mu)}(z) = \Gamma(\alpha) \cdot \lim_{k \to +\infty }  k^{1+\mu-\alpha} a_k.
\EEQ
Thus, to obtain an equivalent of the form $a_k \sim c k^{-\omega}$ for some $c \in \rb^\ast$, we need to obtain equivalents of the $\mu$-th derivative of $A$, with $\mu > \omega -1$ and $\alpha  = 1 + \mu - \omega > 0$, and consider $\ds c = \frac{1}{\Gamma(\alpha)} \lim_{z \to 1^-} (1-z)^{\alpha} A^{(\mu)}(z)$.
A good candidate is $\mu = \lceil \omega \rceil - 1$, but any greater integer $\mu$ suffices.
{In other words, the $\mu$-th derivative can recover asymptotic equivalents of the form
$a_k \sim c k^{-\omega}$ for any $0 < \omega < \mu + 1$; the boundary case
$\omega = \mu + 1$ requires the logarithmic treatment below.}
 A sufficient condition for equivalence can be obtained by applying the earlier result: {the  corresponding Tauberian condition is
$
k^{2-\alpha}
\Big[
\frac{\Gamma(k+\mu+1)}{\Gamma(k+1)}a_{k+\mu}
-
\frac{\Gamma(k+\mu)}{\Gamma(k)}a_{k+\mu-1}
\Big]
\leqslant c.
$
When $a_k\geqslant 0$ (like in our situation), the lower-sign version of this condition is implied by
$
k^{2+\mu-\alpha}(a_k-a_{k-1})\geqslant -c',
$
and hence by
$
k^{2+\mu-\alpha}|a_k-a_{k-1}|=O(1).
$}

In summary, the $z$-transform method consists in computing derivatives of $A$ and looking at their limits when $z$ tends to one. It will always be used in two steps: (1) show that Tauberian conditions hold, and (2) perform an asymptotic expansion of the $z$-transform.

\paragraph{Logarithmic behavior.}
We will also need the $z$-transform of the harmonic series ${h_k} = \sum_{i=1}^k \frac{1}{i}$ (with value $0$ at $k=0$), for which we have, for $|z|<1$~\citep[Section 1.2.9]{knuth1997art}:
\BEQ
\label{eq:log}
\sum_{k=1}^{+\infty} {h_k} z^k =  \sum_{k=1}^{+\infty} \sum_{i=1}^k \frac{1}{i} z^k 
= \sum_{i=1}^{+\infty}  \frac{1}{i}  \sum_{k=i}^{+\infty} z^k
= \sum_{i=1}^{+\infty}  \frac{1}{i}  \frac{z^i}{1-z} = - \frac{ \log(1-z) }{1-z}.
\EEQ
Since ${h_k} = \log (k) + c + o(1)$, where $c$ is the Euler constant, if $(1-z) A(z) = - c' \log(1-z) + c'' + o(1) $ for some constants $c' \in \rb^\ast$, $c'' \in \rb$, then
the sequence $ b_k = a_k -  c'  {h_k} - c''$ is such that its $z$-transform $B$ satisfies $(1-z)B(z) = o(1)$, which implies that $b_k = o(1)$, that is, $a_k =  c'\log(k) + c' c + c'' + o(1) \sim c' \log (k)$.

More generally, when the two limits exist, we have: $$\lim_{z \to 1^-} \frac{(1-z)^\alpha A(z)}{ - \log(1-z)} = \Gamma(\alpha) \lim_{k \to +\infty} \frac{  k^{1-\alpha} a_k}{\log k}.$$

\paragraph{Beyond Tauberian theorems.} In this paper, we consider sufficient conditions for the final value theorem based on conditions on the sequence $(a_k)_{k \geqslant 0}$. Alternative frameworks could also be considered using complex analysis, which only depends on properties of the $z$-transform $A(z)$ for complex $z$~\citep{bender1974asymptotic,flajolet1990singularity}.

\subsection{Classical properties}
\label{sec:classic}

\paragraph{Inversion.} As soon as the sequence $(a_k)_{k \geqslant 0}$ does not grow faster than a polynomial, the $z$-transform $A$ is holomorphic on the open complex unit disc, and from any holomorphic function $A$ defined on the open unit disk, we can obtain the sequence using derivatives $a_k = \frac{1}{k!} A^{(k)}(0)$. Alternatively, we can use Cauchy residue formula \citep[see, e.g.,][]{lang2013complex}
\BEQ
\label{eq:cauchy}
a_k = \frac{1}{2i \pi} \oint_\mathcal{C} A(z^{-1}) z^{k-1} dz,
\EEQ
for any counter-clockwise contour $\mathcal{C}$ around the origin {lying outside the closed unit disk}. This is related to Fourier series, as when the series $(a_k)_{k \geqslant 0}$ is absolutely convergent, the Fourier series is defined
as
$A(e^{-i \omega}) = \sum_{k = 0}^{+\infty} a_k e^{-i k \omega}$, with inverse Fourier transform
$a_k = 
\frac{1}{2i \pi} \int_0^{2\pi}  A(e^{-i \omega}) (e^{i \omega})^{k-1} i e^{i\omega} d\omega $, which is exactly \eq{cauchy} (in all cases, we can take for $\mathcal{C}$ the circle of radius ${r  > 1}$).

\paragraph{Linear difference equations.} If $A(z)$ is the $z$-transform of the sequence $(a_k)_{k \geqslant 0}$, then $z A(z)$ is the $z$-transform of the sequence $(0,a_0,a_1,\dots)$, that is, shifted by $1$. Thus, if the sequence satisfies a difference equation of the form $\forall k \geqslant 0, \ \sum_{i=0}^s \lambda_i a_{k+i}  = 0$, then we have for all $|z|<1$,
$$
0 = \sum_{k = 0}^{+\infty} z^k\sum_{i=0}^s \lambda_i a_{k+i}
= \sum_{i=0}^s \lambda_i z^{-i}\sum_{k = 0}^{+\infty} a_{k+i} z^{k+i}
=\sum_{i=0}^s \lambda_i z^{-i} \Big[ A(z) - \sum_{j=0}^{i-1} a_j z^j \Big]=0.$$
This implies that $A(z)$ is a rational function with denominator $\sum_{i=0}^s \lambda_i z^{s-i}$. Using a partial-fraction decomposition, this implies the classical result regarding the solution of such linear difference equation: assuming that the roots of the denominator are all simple, then we have $A(z) = \sum_{i=1}^s \frac{d_i}{1-c_i z}$ for some complex numbers $c_i,d_i$, leading to $a_k = \sum_{i=1}^s d_i c_i^k$ for all $k \geqslant 0$.
 
 \paragraph{Derivative.}
 If $A$ is the $z$-transform of $(a_k)_{k \geqslant 0}$, then $z \mapsto z A'(z)$ is the $z$-transform of $(k a_k)_{k \geqslant 0}$. Thus, when we have a difference equation with rational coefficients (in $k$), this leads to an ordinary differential equation on the $z$-transform (see \mysec{nesterov-1} and \mysec{nesterov-2} for more details and applications to Nesterov acceleration).

\paragraph{Link with Laplace transform.}  
The $z$-transform is essentially the Laplace transform, transposed from continuous to discrete time. For a function $f: \rb \to \cb$, its Laplace transform is $F(s) = \int_{0}^{+\infty} e^{-st} f(t) dt$. When approximating the integral by a Riemann sum
$ {T}\sum_{k=0}^{+\infty} e^{-s k T} f(kT) =  {T}\sum_{k=0}^{+\infty} (e^{-s T})^k f(kT)$, we obtain the $z$-transform of the sequence $({T}f(kT))_{k \geqslant 0}$ taken at $z = e^{-sT}$. This analogy is classical in control theory and signal processing, and many of the results and techniques for Laplace transforms have their discrete counterparts with the $z$-transform, such as the final value theorem $\lim_{t \to +\infty} f(t) = \lim_{s \to 0^+} s F(s)$  \citep[see, e.g.,][Section 1.13]{korevaar2004tauberian}.

\paragraph{{Exponentially converging sequences.}} {If the sequence $(a_k)_{k \geqslant 0}$ converges exponentially fast to zero, then the $z$-transform $A$ has a radius of convergence strictly greater than $1$, let us say $1/c$ for $c < 1$. Then, the sequence $(c^{-k} a_k)_{k \geqslant 0}$, with $z$-transform $z \mapsto A(z/c)$ can be studied instead with the tools described above to obtain asymptotic behaviors in the form $c^k / k^\omega$, as will happen at the end of \mysec{gd} for strongly-convex problems such as when using regularization.}

\subsection{Convolutions of $z$-transforms}
\label{sec:convolution}
A classical property of $z$-transforms is the following time-domain convolution property relating the product of $z$-transforms of $(a_k)_{k \geqslant 0}$  and $(b_k)_{k \geqslant 0}$, with the $z$-transform of their convolution defined as
$$
(a * b)_k = \sum_{i = 0}^k a_i b_{k-i}.
$$
(This is also the classical multiplication theorem for power series.)

In this paper, we will need to understand the $z$-transform of $(a_k b_k)_{k \geqslant 0}$ from the $z$-transforms $A(z)$ and $B(z)$ of the sequences $(a_k  )_{k \geqslant 0}$ and $( b_k)_{k \geqslant 0}$, in particular for $A=B$. This $z$-transform will be the convolution between $A$ and $B$, defined when we know the $z$-transforms for all complex numbers $z$ such that $|z|<1$. We can then write the Cauchy residue formula in \eq{cauchy}, for $r<1$, as
$$
a_k = \frac{1}{2\pi} \int_{-\pi}^{\pi} r^{-k} A( r e^{-i\omega}) e^{i\omega (k-1) } e^{i\omega} d\omega = \frac{1}{2\pi} \int_{-\pi}^{\pi} r^{-k} A( r e^{-i\omega}) e^{i\omega k} d\omega.
$$
Following~\citet{jury1964theory}, the convolution of $z$-transforms is defined so that $(A\ast B)(e^{-i\omega})$ is the classical $2\pi$-periodic convolution on $[-\pi,\pi]$ of $A(e^{-i\omega})$ and $B(e^{-i\omega})$:
$$
(A * B)(z) = \frac{1}{2i \pi} \oint_\mathcal{C} A(y^{-1}) B(z y ) y^{-1} {dy},
$$
and this is the $z$-transform of $(a_k b_k)_{k \geqslant 0}$.

\paragraph{Properties of convolutions.} The convolution is a bilinear transform and has some nice algebraic properties that are useful when analyzing algorithms that lead to rational functions. The convolution equality that leads to all the ones that we will need is 
$$
\frac{1}{a- uz} * \frac{1}{b- vz} = \frac{1}{ab- uvz},
$$
for all $u$ and $v$ with modulus strictly less than $1$. As shown in Appendix~\ref{app:conv}, using partial-fraction decompositions, all convolutions of rational functions can be obtained.  This allows ones to obtain simple formulas for equivalents and for proving Tauberian sufficient conditions.

 \subsection{Convergence of oscillating sequences}
 In the analysis of accelerated methods, we will look at two related $z$-transforms, namely
 $A(z) = \frac{1}{(1-z)^2 + \lambda  }$ for Nesterov acceleration in \mysec{nesterov-1}, and $B(z)= \frac{1}{(1-z)^2 + \lambda {z}}$ for the heavy-ball method in \mysec{hb}, as well as their squared convolutions $A*A(z)$ and $B*B(z)$. Around $z=1$, they seem equivalent, but their poles (zeros of their denominators as functions of $z$) are different. For $\lambda \in (0,4)$, they both have complex conjugate roots, but their square modulus is $1+\lambda$ for $A$, while it is $1$ for $B$. Thus, oscillations are converging to zero for the corresponding sequence $(a_k)_{k \geqslant 0}$ (and then Abel's theorem holds), while they do not for $B$ (Abel's theorem does not hold).

\section{Gradient descent}
\label{sec:GD}
\label{sec:gd}
We consider a compact positive semidefinite operator  $H : \H \to \H$ on a separable Hilbert space of infinite dimension. Gradient descent for the minimization of $F(\eta) = \frac{1}{2} \langle \eta-\eta_\ast, H ( \eta-\eta_\ast)\rangle$ is equivalent to the following \emph{linear} iteration 
$$
\eta_k = \eta_{k-1} - \gamma F'(\eta_{k-1}),
$$
where $\gamma > 0$ is a stepsize which we choose so that the operator norm of $\gamma H$ is less than one. With 
 $\theta_k = \eta_k - \eta_\ast$, this leads to the following iteration
\BEQ
\label{eq:GDrec} \theta_k = ( \idm - \gamma H) \theta_{k-1} = ( \idm - \gamma H)^k \delta,
\EEQ
where {$\delta$ is equal to  $\theta_0 =  \eta_0 - \eta_\ast$}. The performance guarantee is often taken to be the difference in function values
$$
a_k = F(\eta_k) - F(\eta_\ast) = \frac{1}{2}\langle \theta_k,H \theta_k \rangle =  \frac{1}{2\gamma}\langle \delta, ( \idm - \gamma H)^{2k}  \gamma H  \delta \rangle.
$$
{In order to define function values, we do not necessarily need that $\delta = \theta_0$ belongs to~$\H$ (in fact we will consider situations where the norm is infinite), but only that $H^{1/2} \delta_0$ is well-defined in $\H$. For this paper, as done below, we will  simply assume that $\delta$ has a decomposition in the basis of eigenvectors of $H$, that may or may not be square-integrable.}
Thus, given a spectral decomposition $\gamma H = \sum_{i=1}^{+\infty} \lambda_i u_i \otimes u_i$ with eigenvalues $(\lambda_i)_{i \geqslant 1}$ in $[0,1]$ and eigenvectors $(u_i)_{i \geqslant 1}$ in $\H$ (which exist since $H$ is compact), we get
$$
a_k = \frac{1}{2\gamma} \sum_{i=1}^{+\infty} \lambda_i b_k( \lambda_i)^2  \langle u_i,{\delta} \rangle^2
= \int_{\rb_+}b_k(\lambda)^2 d\sigma(\lambda),
$$
where 
\BEQ
\label{eq:AA}
d\sigma(\lambda) =  \frac{1}{2\gamma} \sum_{i=1}^{+\infty} \lambda_i  \langle u_i,{\delta} \rangle^2 {\rm Dirac}(\lambda|\lambda_i)
\EEQ
is a weighted spectral measure, as defined by \citet{berthier2020accelerated}, and $b_k(\lambda) = ( 1- \lambda)^{k}$,
with a $z$-transform equal to 
 $$
 A(z) = \int_0^{1} B(z, \lambda) * B(z,\lambda)  d \sigma(\lambda),
 $$
 with a convolution with respect to the first variable (as defined in \mysec{convolution}).
 
 \paragraph{Interpretation with $z$-transform of vector-valued sequences.}
We have a sequence $(\theta_k)_{k \geqslant 0}$ in a Hilbert space $\H$, and we can also define its $z$-transform as
$$
\Theta(z) = \sum_{k=0}^{+\infty} z^k \theta_k \in \H.
$$
We are here interested in quadratic forms in $\theta_k$, that is, in the $z$-transform of $a_k = \frac{1}{2} \langle \theta_k, H \theta_k \rangle$ for some positive semidefinite compact self-adjoint operator $H$, and our method works because $\theta_k$ is equal to $b_k(H)$ for some spectral function $b_k$.

 \paragraph{Asymptotic convergence.}
 
 In the particular case of gradient descent with a stepsize that is less than one over the largest eigenvalue of the Hessian operator $H$, the sequence $(a_k)_{k \geqslant 0}$ is non-increasing, so all Tauberian theorems from \mysec{tauber} apply.
  
 We then have, explicitly, using partial-fraction decompositions (seen as a function of $\lambda$):
 \BEA
 \label{eq:AAA}
A(z) & = & \sum_{k=0}^{+\infty} z^k  \int_0^{1  } (1-\lambda)^{2k} d\sigma(\lambda) 
=   \int_0^{1  } \frac{1}{1 - z (1-\lambda)^{2}} d\sigma(\lambda)  \\
\notag & = & \frac{1}{2}
 \int_0^{1  }\bigg( \frac{1}{1 - \sqrt{z}  (1-\lambda)} +  \frac{1}{1 + \sqrt{z}  (1-\lambda)}  \bigg) d\sigma(\lambda)  
\\
\notag & = & \frac{1}{2 \sqrt{z} }
 \int_0^{1  }  \frac{d\sigma(\lambda)  }{   \lambda + z^{-1/2} - 1 } - 
  \frac{1}{2\sqrt{z}}
 \int_0^{1  }  \frac{d\sigma(\lambda)  }{   \lambda - z^{-1/2} - 1 }. \EEA
We thus get
\BEQ
\label{eq:AGD}
A(z)  = \frac{1}{2 \sqrt{z} } S(  z^{-1/2}-1 ) - \frac{1}{2 \sqrt{z}}S ( - z^{-1/2}-1) ,
\EEQ
where $\ds S(u) = \int_0^1 \frac{d\sigma(\lambda)}{\lambda + u}$ is the \sti transform of the measure $\sigma$ \citep[see][Chapter 8]{widder1942laplace}.
Thus, the behavior of $A(z)$ and its derivatives for $z$ tending to $1^-$ can be obtained from the behavior of $S(u)$ and its derivatives when $u$ tends to $0^+$ and $-2$.

Since we have assumed that the support of $\sigma$ is included in $[0,1]$, all derivatives of $S$ at $u=-2$ are finite. We thus need to use a characterization of the exploding asymptotic behavior of $S(u)$ and its derivatives (since we require derivatives of $A$ to characterize fast decays) around $u=0^+$.

We will make the following assumption:
\BIT
\item[\textbf{(A1)}] The spectral measure $\sigma$ is assumed to be supported within $[0,1)$, and there exist
$\omega>0$ and $c>0$ such that for any integer $k$ strictly larger than  $\omega$, around $u=0^+$: 
\BEQ
\label{eq:sti}
S^{(k-1)}(u) \sim c \cdot  (-1)^{k-1}  {\Gamma(k-\omega) \Gamma(\omega)} u^{\omega - k }.
\EEQ
The real number $\omega >0$ is referred to as the spectral dimension.
\EIT
The following lemmas give examples of such spectral measures \citep[see][for other examples, in particular in distributed optimization]{berthier2020accelerated}. Note that \citet{berthier2020accelerated} consider a weaker version of spectral dimension, that is,  $\sigma((0,u)) \sim_{u \to 0^+} \frac{c}{\omega}  u^\omega$, which is implied by Assumption~\textbf{(A1)} (see Appendix~\ref{app:sdb}).

\begin{lemma}[Continuous spectral measure]
\label{lemma:conti}
If  $d\sigma(\lambda) = c 1_{[0,\mu]} \lambda^{\omega-1} d\lambda$ for $\mu \in (0,1)$, for $\omega > 0$, then $\sigma$ satisfies assumption \emph{\textbf{(A1)}}.
\end{lemma}
\begin{proof}
We have, using a change of variable,
$$
S^{(k-1)}(u) = c (k-1)!  (-1)^{k-1} \int_0^\mu \frac{\lambda^{\omega - 1}}{(\lambda +u)^{k}} d\lambda
= c (k-1)!  (-1)^{k-1} u^{\omega -k} \int_0^{\mu/u}  \frac{\lambda^{\omega - 1}}{(\lambda +1)^{k}} d\lambda.
$$ 
Since $\int_0^{+\infty}  \frac{\lambda^{\omega - 1}}{(\lambda +1)^{k}} d\lambda = 
\int_0^{+\infty}  \big( \frac{\lambda}{\lambda+1}\big)^{\omega-1} \big( \frac{1}{\lambda+1}\big)^{k-\omega-1}  \frac{d\lambda}{(1+\lambda)^2} = \int_0^1 (1-u)^{\omega-1} u^{k - \omega-1} du$ is equal to 
$\frac{\Gamma(k-\omega) \Gamma(\omega)}{\Gamma(k)}$, we get the desired equivalent.
\end{proof}
\begin{lemma}[Discrete spectral measure]
\label{lemma:discrete}
If  $\sigma(\lambda) = \sum_{i=1}^{+\infty} \frac{1}{i^{\alpha+\beta}} {\rm Dirac}(\lambda | \mu i^{-\alpha})$, for $\alpha, \beta > 0$ such that $\beta > 1- \alpha $, then $\sigma$ satisfies assumption \emph{\textbf{(A1)}} with $\omega =  \frac{\beta-1}{\alpha} + 1 > 0$ and $c = \frac{1}{\alpha \mu^{\frac{\beta-1}{\alpha}+1}}$.
\end{lemma}
\begin{proof}
We have, using an integral-series comparison (which is here an asymptotic equivalent because we have rational functions), around $u=0^+$:
 $$
S^{(k-1)}(u) = (-1)^{k-1} (k-1)! \sum_{i=1}^{+\infty} \!\frac{1}{i^{\alpha+\beta}} \frac{1}{(u+ \mu i^{-\alpha})^k}
\sim (-1)^{k-1} (k-1)! \int_1^{+\infty}\!\!\!  \frac{1}{t^{\alpha+\beta}} \frac{1}{(u+ \mu t^{-\alpha})^k} dt.
$$
With the change of variable $\lambda = \mu t^{-\alpha}$, we get the desired result from Lemma~{\ref{lemma:conti}}.
\end{proof}
With the lemma above, for gradient descent, the following lemma shows that Assumption \textbf{(A1)} is satisfied.
\begin{lemma}[Spectral dimension for gradient descent]
\label{lemma:GD}
Assume $\langle \delta, u_i \rangle = \frac{\Delta}{i^{\beta/2}}$ and $\lambda_i = \frac{\gamma L}{i^\alpha}$, {with $\alpha > 0$, $\beta > 1-\alpha$ and $\gamma L \in (0,1)$}, then, for the measure $\sigma$ defined in \eq{AA}, Assumption \textbf{(A1)} is satisfied for $\omega = \frac{\beta-1}{\alpha} + 1$
and $c = \frac{L \Delta^2 }{2 \alpha (\gamma L)^{\frac{\beta-1}{\alpha}+1}}.$
\end{lemma}

 We can now prove the asymptotic convergence rate for gradient descent.

\begin{proposition}[Asymptotics for gradient descent]
\label{prop:gd}
For the sequence $(a_k)_{k \geqslant 0}$ defined through its $z$-transform $A$ in \eq{AGD}, assuming \emph{\textbf{(A1)}}, we have, when $k \to +\infty$,
$$
a_ k  \sim c  \frac{\Gamma(\omega)}{(2k)^{\omega}}.
$$
\end{proposition}
\begin{proof}
When $\omega \in (0,1)$, Assumption \textbf{(A1)} leads
to, when $z \to 1^-$, $$
A(z)   \sim \frac{c}{2^\omega} ( 1- z)^{\omega -1 } \Gamma(\omega) \Gamma(1-\omega),$$ which leads to the equivalent $a_k \sim c  \frac{\Gamma(\omega)}{(2k)^{\omega}}$ from Theorem~\ref{theorem:gamma} (which applies since the sequence $(a_k)_{k \geqslant 0}$ is non-increasing).

In order to obtain the equivalent for all $\omega > 0$, we need to consider the derivative of $A$ defined in \eq{AAA}:
\BEAS
A^{(\mu)}(z)
& = & \mu!  \int_0^{1  } \frac{(1-\lambda)^{2\mu}}{ ( 1 - z (1-\lambda)^{2}  )^{1+\mu}} d\sigma(\lambda).
\EEAS
In the partial fraction decomposition (as a function of $\lambda$), the term that will contribute the most to the asymptotics in $z\to 1^-$ is
$$ \frac{ \mu!}{2^{1+\mu}}
  \int_0^{1  } 
 \frac{(1-\lambda)^{2\mu}}{ ( 1 - \sqrt{z}  (1-\lambda)  )^{1+\mu}}d\sigma(\lambda) 
 = \frac{ \mu!}{(2 \sqrt{z})^{1+\mu}}
  \int_0^{1  } 
 \frac{(1-\lambda)^{2\mu}}{ ( \lambda + z^{-1/2} - 1 )^{1+\mu}}d\sigma(\lambda) . $$
We then use the equivalent $S^{(\mu)}(u) \sim c {(-1)^\mu} u^{\omega - \mu - 1} \Gamma ( \omega + 1 - \mu) \Gamma(\omega)$, and thus, since the term $(1-\lambda)^{2\mu}$ {can be replaced by $1$ when $\lambda \to 0$},
$$
A^{(\mu)}(z) \sim \frac{c}{2^\omega} (1-z)^{\omega - \mu - 1} \Gamma ( - \omega + 1 + \mu) \Gamma(\omega),
$$
which is valid for $\omega \in (0,\mu+1)$. From {the derivative extension of Theorem~\ref{theorem:gamma} derived in \eq{derivative}}, this leads to the desired equivalent for all $\omega > 0$, as
$$ \lim_{k \to + \infty} k^\omega a_k = \frac{\lim_{z \to 1^-} A^{(\mu)}(z)(1-z)^{\mu +1 - \omega}}{\Gamma(1+\mu-\omega)}
.$$
\end{proof}
When applied to gradient descent, that is, with $\omega = \frac{\beta-1}{\alpha} + 1$
and $c = \frac{L \Delta^2 }{2 \alpha (\gamma L)^{\frac{\beta-1}{\alpha}+1}}$ from Lemma~\ref{lemma:GD}, we can make the following observations:
\BIT
\item {For the simple case of gradient descent, the expression of the $z$-transform in \eq{AGD} was obtained through a closed-form expression of the performance $(a_k)_{k \geqslant 0}$. We will see in subsequent sections that, owing to the properties outlined in \mysec{classic} and \mysec{convolution}, such closed forms are not necessary to study the asymptotics of the $z$-transform.}
\item The same result can also be obtained from applying Laplace's method directly on the sequence $(a_k)_{k \geqslant 0}$ (see {\small \url{https://francisbach.com/scaling-laws-of-optimization/}}); this will not be the case for subsequent results.
\item The asymptotic result should be compared with the traditional worst-case result which leads to an upper bound proportional to  $ {L\| \delta\|^2 }/{k} $~\citep[see, e.g.,][]{nesterov2018lectures}. These two results are not contradictory, as the worst-case rate in $1/k$ relies on $\|\delta\|^2$ being finite, which is only true when $\beta>1$, which implies $\omega>1$ and thus a rate faster than $1/k$. The result from Proposition~\ref{prop:gd} is thus an (asymptotic) improvement that relies on a finer characterization of the spectrum of the Hessian and the coordinates of the initial deviation to the optimum.
\item The decays characterized by $\alpha$ and $\beta$ are naturally achieved in statistics and machine learning with square loss and predictors that are linear functions of some feature vectors $\varphi(x)$, where the Hessian is the empirical second-order moment of $\varphi(x)$, which tends to the second-order moment by the law of large numbers when the number of observations goes to infinity. These decays correspond to ``source'' and ``capacity'' conditions~\citep{nemirovsky1992information,caponnetto2007optimal,dieuleveut2017harder}. For example, when predicting a function defined on $\rb^d$ belonging to a Sobolev space of order $t>0$ on a compact regular set with a model corresponding to functions that are in a Sobolev space of order $s>d/2$, we have $\alpha = 2s / d > 1$ and $\beta = 2t/d$ \citep[see][for more details on the relevance of such assumptions]{caponnetto2007optimal,velikanov,bach2024learning}. See also~\cite{berthier2020accelerated} for characterizations of the spectral dimension in gossip algorithms.
 \EIT

\paragraph{Randomized initial condition.} We may consider a random vector $\delta \in \H$ characterizing initial conditions, where we assume that
$ \langle \delta, u_i \rangle = \frac{\Delta}{i^{\beta/2}} r_i$, where each $r_i$ for $i\geqslant 1$ is an independent random variable with 
$\E[ r_i^2] = 1$ and $\E [r_i^4]$ equal to a constant $\kappa+3$, where $\kappa \geqslant -2$ is the excess kurtosis (Proposition~\ref{prop:gd} corresponds to $r_i$ almost surely equal to $1$).
We can then show same asymptotic equivalent for $\E [a_k]$ (as this is exactly the previous result in  Proposition~\ref{prop:gd}), and we have:
\BEAS
a_k - \E[a_k] & = & \sum_{i = 1}^{+\infty}
\frac{L \Delta^2}{2} \Big( 1 - \frac{\gamma L}{i^\alpha} \Big)^{2k} \frac{1}{i^{\alpha + \beta}} (r_i^2 - 1)
\\
{\rm var}(a_k)
& = & 
\sum_{i = 1}^{+\infty}
\frac{L^2 \Delta^4}{4} \Big( 1 - \frac{\gamma L}{i^\alpha} \Big)^{4k} \frac{1}{i^{2\alpha + 2\beta}}(\kappa+2).
 \EEAS
 Thus,
 \BEAS
\sum_{k=0}^{+\infty}  {\rm var}(a_k) z^k 
& = & 
\sum_{i = 1}^{+\infty}
\frac{L^2 \Delta^4}{4} \frac{1}{ 1 - z \big( 1 - \frac{\gamma L}{i^\alpha} \big)^{4}} \frac{1}{i^{2\alpha + 2\beta}}(\kappa+2) = \int_0^1\frac{1}{1-z(1-\lambda)^4} d\sigma'(\lambda),
 \EEAS
 with $\sigma'$ a spectral measure satisfying Assumption \textbf{(A1)} with $c'=\frac{L^2 \Delta^4}{4 \alpha (\gamma L)^{\omega'}}(\kappa+2)$ and $\omega' = 2 \omega + \frac{1}{\alpha}$. Thus,
 $\E[ a_k ]  \sim c \frac{\Gamma(\omega)}{(2k)^{\omega}}$, and
 $ {\rm var}(a_k) \sim c' \frac{\Gamma(\omega')}{(4k)^{\omega'}}$.
 Since $\omega' > 2\omega$, we have $ \sqrt{ {\rm var}(a_k)} = o( \E[a_k] )$, and 
  we obtain the same equivalent as deterministic sequences, but now \emph{in probability} \citep[see][Section 2.2]{van2000asymptotic}.

\paragraph{Alternative criterion.} We may measure performance in terms of distance to optimum in Hilbert norm, that is, $\| \theta_k\|^2$. A simple extension to Proposition~\ref{prop:gd} leads to the equivalent $\|\theta_k\|^2 \sim \frac{ \Delta^2}{\alpha ({2} \gamma L)^{\frac{\beta-1}{\alpha}}}\frac{{\Gamma(\frac{\beta-1}{\alpha})}}{k^{\frac{\beta-1}{\alpha}}}$, which is, as expected, slower.

\paragraph{{Strongly-convex problems.}} {In this paper, we primarily focus on sublinear convergence rates of the form $1/k^\omega$, but the $z$-transform can treat equally well the presence of exponential behavior, which essentially shifts the spectral measure of eigenvalues away from $0$, let us say by $\nu > 0$ (for example, when adding the regularizing term $\frac{\nu}{2} \| \theta\|^2$ to a quadratic objective function). This shifts the pole of the \sti transform from $0$ to $-\nu$. Thus, from \eq{AGD}, the pole of the $z$-transform $A$ occurs at $1/c$ such that $c^{1/2}-1 = -\nu$, that is, $c = (1-\nu)^2$, leading to, as seen at the end of \mysec{classic}, a convergence of the form $(1-\nu)^{2k}$ times a term in $1/k^\omega$ that can be computed with the same tools as for non-strongly-convex problems.}

\section{Nesterov acceleration}
\label{sec:nesterov-1}
A classical way to accelerate the convergence of the gradient iteration is to use an extrapolation step due to \citet{nesterov1983method}. This can be formulated with two sequences (and a first-order recursion) as:
\BEAS
\eta_{k+1} & = &  \zeta_{k} - \gamma F'(\zeta_{k}) \\
\zeta_{k+1} & = &  \eta_{k+1} + ( 1 - \tau_{k+1}) ( \eta_{k+1} - \eta_k),
\EEAS
and with a single sequence (and a second-order recursion) as
$$
\eta_{k+1} = \eta_{k} + ( 1 - \tau_{k}) ( \eta_{k} - \eta_{k-1})
- \gamma F' \big[ \eta_{k} + ( 1 - \tau_{k}) ( \eta_{k} - \eta_{k-1})\big].
$$
In the original formulation with convex functions, $\tau_k$ is chosen asymptotically proportional to $3/(k+1)$. For quadratic functions, \citet{flammarion2015averaging} argue that the choice $2/(k+1)$ is more natural, which we choose in this section (see \mysec{nesterov-2} for a more general situation). As in \mysec{GD}, we consider the function
$F(\eta) = \frac{1}{2} \langle \eta-\eta_\ast, H ( \eta-\eta_\ast)\rangle$, leading to, with $\theta_k = \eta_k - \eta_\ast$,  the iteration studied by \citet{flammarion2015averaging}:
\BEA
\label{eq:nesterov-1}
 \theta_{k+1}  &  = &  ( \idm - \gamma H) \Big[  \theta_{k} + \Big( 1 - \frac{ 2  }{k+1} \Big) ( \theta_{k} - \theta_{k-1}) \Big] \\
\notag &  = &  ( \idm - \gamma H) \Big[  2   \frac{ k }{k+1}    \theta_{k} - 
  \frac{ k- 1 }{k+1}  \theta_{k-1} \Big] ,
\EEA
initialized with $\theta_1 = \theta_0$. Following \citet{flammarion2015averaging}, the equivalent iteration for  $\xi_k = k \theta_{k}$ is 
$$
\xi_{k+1} = ( \idm - \gamma H) ( 2\xi_k - \xi_{k-1}),
$$ 
with $\xi_0 = 0$ and $\xi_1 = \theta_0$. Instead of considering a recursion for the vector $ { \xi_{k} \choose \xi_{k-1}}$ \citep[as performed by][]{flammarion2015averaging}, we use properties of the $z$-transform. Indeed, we have, for $\Xi$ the associated $z$-transform of the vector-valued sequence $(\xi_k)_{k \geqslant 0}$, using that shifts correspond to multiplication by $z$:
$$
\Xi(z) - z\xi_1 = ( \idm - \gamma H)  (2z - z^2)
\Xi(z),
$$
leading to, by solving the linear system:
$$
\Xi(z) = z\big( \idm -  (2z - z^2) ( \idm - \gamma H)\big)^{-1} \theta_0.
$$
We consider the performance criterion 
$F(\eta_k) - F(\eta_\ast) = \frac{1}{2} \langle \theta_k, H \theta_k \rangle=
\frac{1}{2 k^2 } \langle \xi_k, H \xi_k \rangle = a_{k-1}/k^2$, if we define~$a_k$ as
\BEQ
\label{eq:perf-nesterov-1}
a_k = \frac{1}{2} \xi_{k+1}^\top H \xi_{k+1} = \int_0^1
b_k(\lambda)^2
d\sigma(\lambda),
\EEQ
with the sequence $b_k(\lambda)$ with $z$-transform:
\BEQ
\label{eq:Bnest} B(z,\lambda) = \frac{1}{ 1 - (2z-z^2)(1-\lambda)}= \frac{1}{(1-\lambda) (1-z)^2 + \lambda } . 
\EEQ
There are now two results to prove: (1) what is the equivalent of the $z$-transform $A(z)$ of $a_k$, and (2) is the oscillatory behavior of $(a_k)_{k \geqslant 0}$ compatible with the use of the $z$-transform method?

We can now use properties of convolutions of $z$-transforms presented in \mysec{convolution}: $B(z,\lambda) * B(z,\lambda)$ will be a rational function of $\lambda$ and can thus be expressed as sums of terms of the form
$\frac{b(z)}{(\lambda + a(z))^k}$ with $a$ and~$b$ rational functions in $\sqrt{z}$, thus leading to sum of \sti transforms and derivatives taken at rational functions of $\sqrt{z}$. We can then use equivalents from Assumption \textbf{(A1)} to get the behavior of the $z$-transform around $z=1$.  The following proposition provides the precise equivalent (see detailed proof in Appendix~\ref{app:nesterov-1}, with accompanying Mathematica notebooks to check algebraic calculations).

\begin{proposition}[Nesterov acceleration]
\label{prop:nesterov-1}
{Assume the power-law spectral model of Lemma~\ref{lemma:GD}, with $0<\gamma L <1$}. Then, for the sequence defined in \eq{nesterov-1} and in \eq{perf-nesterov-1}, we have:
$$
\frac{a_k}{k^2} \sim \left\{
\begin{array}{ll}
\ds
 c \frac{\Gamma(1-\omega) \Gamma(\omega)}{\Gamma(3-2\omega)}  
 \frac{1}{2} \frac{1}{4^{\omega-1}} \cdot \frac{1}{k^{ 2\omega}} & \mbox{ if } \omega \in (0,1)\\[.2cm]
  \ds \frac{c}{2} \frac{ \log k}{k^2} & \mbox{ if } \omega = 1
\\[.2cm]
 \ds \frac{c}{2}  \Gamma(\omega-1) \frac{1}{k^{\omega+1}}  & \mbox{ if } \omega > 1.
\end{array}
\right.
$$
\end{proposition}
With the constants $\omega = \frac{\beta-1}{\alpha} + 1$
and $c = \frac{L \Delta^2 }{2 \alpha (\gamma L)^{\frac{\beta-1}{\alpha}+1}}$ from Lemma~\ref{lemma:GD}, we can make the following observations:
\BIT
\item We indeed see a strict acceleration compared to gradient descent, with a rate that is essentially proportional to
$1/k^{\min\{\omega+1,2\omega\}}$ instead of $1/k^\omega$. This rate will be improved in \mysec{nesterov-2} when $\omega>1$.
\item Worst-case bounds are proportional to $L\| \delta\|_2^2 / k^2$, and correspond to the situation when $\beta>1$ and $\omega > 1$, so our asymptotic result is an asymptotic improvement that provides a finer scaling law under additional assumptions.
\item We could also look at the convergence in iterate norm.
\EIT

\paragraph{Experiments.} We provide in \myfig{nesterov} an experimental illustration for various spectral dimensions $\omega$, showing the oscillatory behavior of Nesterov acceleration for $\omega >1$ (but still with convergence as oscillations vanish for large numbers of iterations).

\begin{figure}
\begin{center}
\includegraphics[width=13cm]{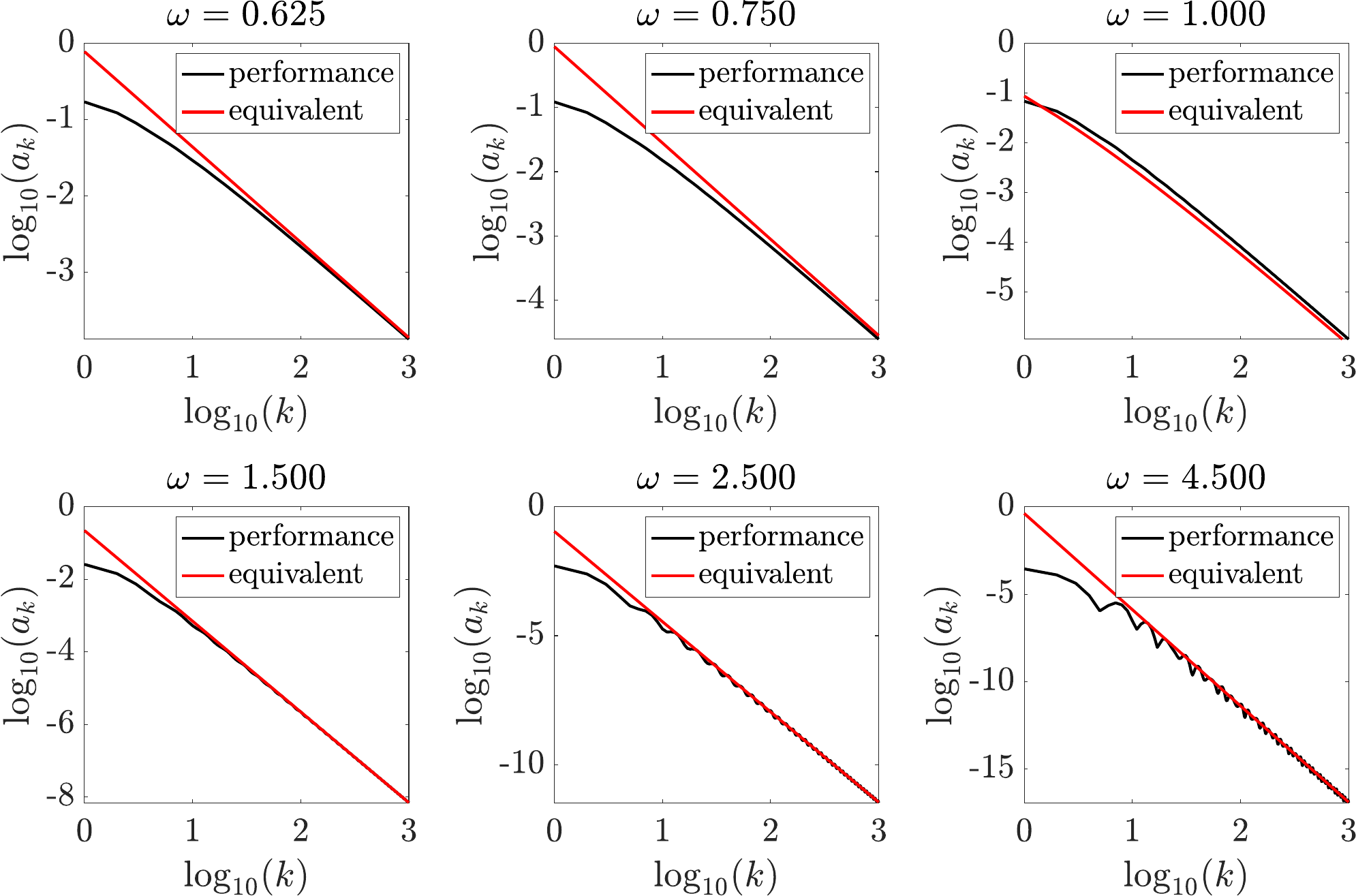}
\end{center}

\vspace*{-.5cm}

\caption{Nesterov acceleration (corresponding to \eq{nesterov-1}) for various spectral dimensions $\omega$, for $\omega<1$, where the convergence rate is proportional to $1/k^{2\omega}$ and for $\omega >1$, where the convergence rate is proportional to  
$1/k^{\omega+1}$ (it is proportional to $\log (k) / k^2$ for $\omega = 1$).
\label{fig:nesterov}} 
\end{figure}

\subsection{Heavy-ball}
\label{sec:HB}
\label{sec:hb}
An alternative to Nesterov acceleration is the heavy-ball method, with an iteration of the form
$$
\eta_{k+1} = \eta_{k} - \gamma F'(\eta_k) + (1-\tau_{k+1}) ( \eta_k - \eta_{k-1}).
$$
As done by \citet{flammarion2015averaging}\footnote{This corresponds to Eq.~(5) by~\citet{flammarion2015averaging} with $\alpha = \gamma$ and $\beta=0$, with a missing term $\frac{n}{n+1}$ in the latter formula on page 4 of their paper.} for least-squares regression, we consider the iteration:
\BEA
\label{eq:ZZ} \theta_{k+1}  &  = &   \Big( \idm - \frac{k}{k+1} \gamma H\Big)  \theta_{k} + \Big( 1 - \frac{ 2   }{k+1} \Big) ( \theta_{k} - \theta_{k-1}), \EEA
with $\theta_1 = \theta_0 = \eta_0 - \eta_\ast$.
With $\xi_k = k \theta_k$, this leads to an iteration with constant coefficients:
$$
\xi_{k+1} = -\gamma H \xi_k + (2 \xi_k - \xi_{k-1})
= ( 2 \idm - \gamma H) \xi_k - \xi_{k-1}.
$$
As for Nesterov acceleration in \eq{nesterov-1}, we can get an equation for the $z$-transform $\Xi$ of the vector sequence $(\xi_k)_{k\geqslant 0}$ as
$$
\big[ \idm - z ( 2\idm -\gamma H) + z^2 \idm \big] \Xi(z) = z \theta_1.
$$
Thus, the criterion is $\frac{1}{2} \theta_{k+1}^\top H \theta_{k+1} = \frac{1}{2(k+1)^2} \xi_{k+1}^\top  H \xi_{k+1}$, with
$$
a_k = \frac{1}{2} \xi_{k+1}^\top H \xi_{k+1} = \int_0^1
b_k(\lambda)^2
d\sigma(\lambda)
\ \ \mbox{ where } \ \ B(z,\lambda) = \frac{1}{ (1-z)^2 + \lambda z}  . $$
This leads to\footnote{See Mathematica notebook {\url{https://www.di.ens.fr/~fbach/ztf/heavyball.nb}}.}
\BEAS
B(z,\lambda) 
* B(z,\lambda) 
&\!\!\!=\!\!\!&\frac{1}{2 \left(\sqrt{z}-1\right) \left(\sqrt{z}+1\right) \left(\lambda  \sqrt{z}-z-2 \sqrt{z}-1\right)} \\
& & \hspace*{1cm} -\frac{1}{2 \left(\sqrt{z}-1\right) \left(\sqrt{z}+1\right) \left(\lambda  \sqrt{z}+z-2 \sqrt{z}+1\right)}
\\
&&\hspace*{1cm} +\frac{2-\sqrt{z}}{4 \left(\sqrt{z}-1\right) \left(\lambda  \left(z-2 \sqrt{z}\right)-z+2 \sqrt{z}-1\right)},
\EEAS
and thus (with equivalents when $z$ tends to 1):
\BEAS
A(z) & = & \frac{S\Big(-\frac{\left(\sqrt{z}+1\right)^2}{\sqrt{z}}\Big)}{2 (z-1) \sqrt{z}}+\frac{S\big(\sqrt{z}+\frac{1}{\sqrt{z}}-2\big)}{(2-2 z) \sqrt{z}}
 \sim \frac{S(-4)}{2 (z-1)}-\frac{S\left(\frac{1}{4} (z-1)^2\right)}{2 (z-1)}.
\EEAS
For the $z$-transform, this leads to the same equivalent as for Nesterov acceleration for $\omega \in (0,1)$. For $\omega > 1$, then we get an equivalent of $\frac{S(0) - S(-4)}{2} \frac{1}{1-z}$ suggesting convergence to a constant times $\frac{S(0) - S(-4)}{2} $. For $\omega = 1$, we obtain the equivalent $c \frac{ \log k}{k^2}$. However, contrary to Nesterov acceleration, oscillations do not vanish when $\omega \geqslant 1$, as illustrated in \myfig{hb} (right plots). It would be useful to investigate the link with the lack of accelerated convergence guarantees of the heavy-ball method for convex functions~\citep[see][and references therein]{goujaud2023provable}, in particular in terms of convergence of Ces\`aro means \citep{ghadimi2015global}, which is common within the Tauberian theory \citep[Section I.6]{korevaar2004tauberian}.

Note that in some particular setups, it is possible to obtain bounds or equivalents for iterations of a similar nature obtained from orthogonal polynomials~\citep{pagliana2019implicit,berthier2020accelerated,berthier2022acceleration,velikanov}. Obtaining conditions under which we can derive general, finer results is left for future work.

\begin{figure}
\begin{center}
\includegraphics[width=13cm]{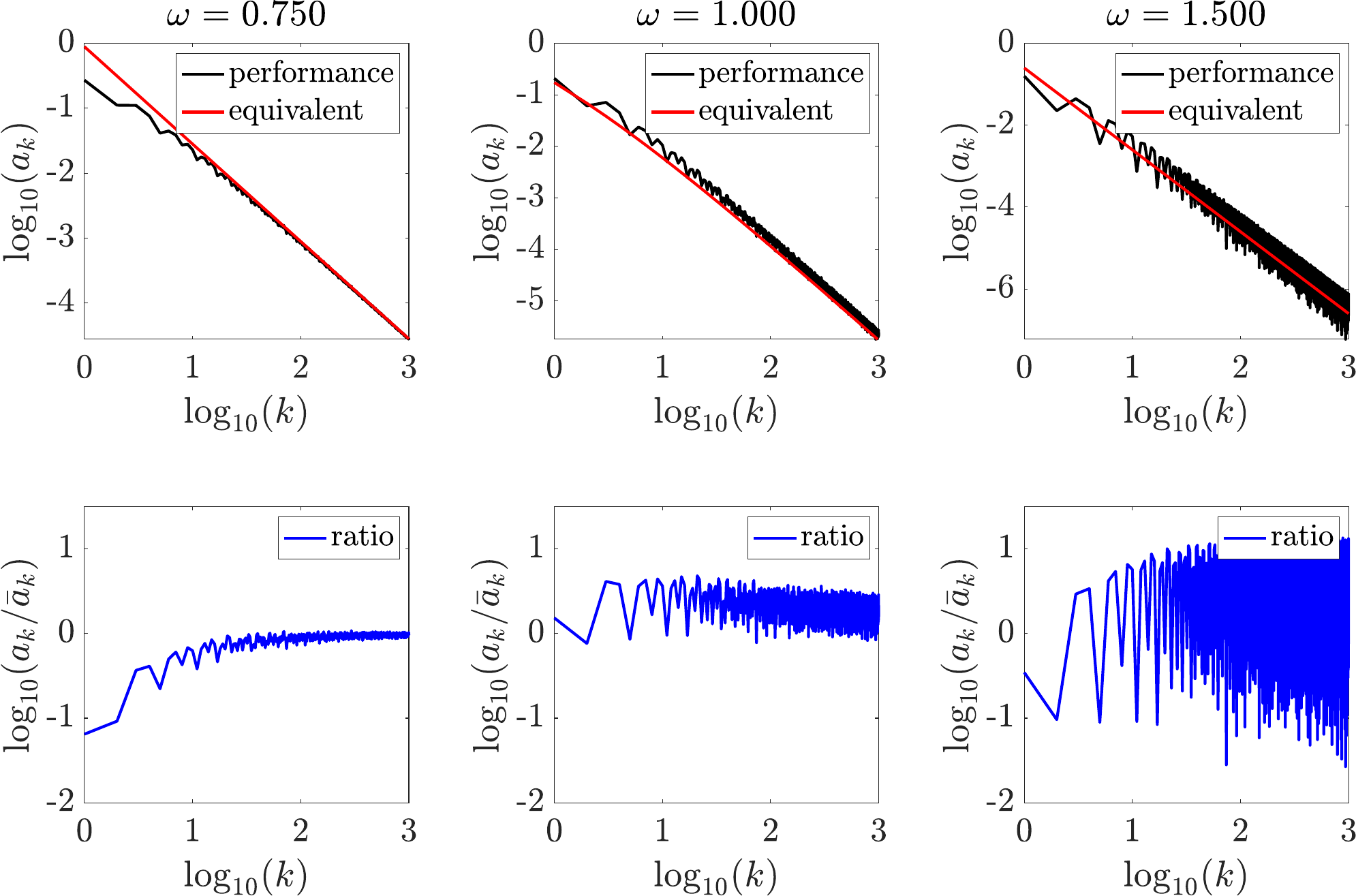}
\end{center}

\vspace*{-.35cm}

\caption{Heavy-ball acceleration: true performance $a_k$ vs. potential asymptotic equivalent  $\bar{a}_k$ (top); ratio (bottom), for several values of the spectral dimension $\omega$.
\label{fig:hb}}
\end{figure}

\subsection{Extended analysis}
\label{sec:nesterov-2}
As studied by~\citet{chambolle2015convergence,aujol2024strong} for upper bounds on the convergence rates for convex functions, we consider the same recursion as in \eq{nesterov-1}, but now with $\tau_k =   2\rho /(k+2\rho-1)$,
leading to for $k \geqslant 1$,
\BEA
\label{eq:nesterov-2}
 \theta_{k+1}  &  = &  ( \idm - \gamma H) \bigg[  \theta_{k} + \Big( 1 - \frac{ 2\rho  }{k+2\rho-1} \Big) ( \theta_{k} - \theta_{k-1}) \bigg] \\
\notag &  = &  ( \idm - \gamma H) \bigg[  2   \frac{ k+\rho-1 }{k+2\rho-1}    \theta_{k} - 
  \frac{ k -1 }{k+2\rho-1}  \theta_{k-1} \bigg] .
\EEA
This exactly extends the case $\rho=1$ treated earlier. We now have for  any $k \geqslant 1$:
\BEAS
(k+2\rho - 1)\theta_{k+1} & = & 
 ( \idm - \gamma H) \big[  2  (  k+\rho-1)     \theta_{k} - 
 (k - 1)   \theta_{k-1} \big]\\
 (k+1)\theta_{k+1} + 2 (\rho-1) \theta_{k+1} & = & 
 ( \idm - \gamma H) \big[  2 k \theta_k  - 
 (k-1)\theta_{k-1} + 2 (\rho-1)  \theta_k   \big].
\EEAS

Multiplying by $z^k$ and summing for $k \geqslant 1$, this leads to the following ODE, since $\Theta'(z) = \sum_{k=1}^{+\infty} k z^{k-1} \theta_k$,
$$
\Theta'(z) - \theta_1
+ 2 (\rho-1) \frac{1}{z} \big[ \Theta(z) - \theta_0 - z \theta_1\big]  = ( \idm - \gamma H) \big[
 (2z - z^2) \Theta'(z) + 2 (\rho-1)( \Theta(z) - \theta_0 )
\big],
$$
that is, since $\theta_1=\theta_0$,
$$
\big[ \idm - (2z -z^2)  ( \idm - \gamma H) \big] \Theta'(z)  + 2 ( \rho-1 ) \big[\Theta(z) - \theta_0 \big]
\Big[ \frac{1}{z} \idm - ( \idm - \gamma H) \Big]
=    (2 \rho-1)  \theta_1.
$$
With  $B(0,\lambda)=1$ and
\BEQ
\label{eq:perf-nesterov-2}
a_k = \frac{1}{2} \theta_k^\top H \theta_k = \int_0^1
b_k(\lambda)^2
d\sigma(\lambda),
\EEQ
we get an ordinary differential equation for $B(z,\lambda)$:
\BEQ
\label{eq:nesterov-22} \big[ 1 - (2z -z^2)  ( 1 - \lambda) \big]  B'(z,\lambda)  + 2 (\rho-1)  [ B(z,\lambda)-1]
\cdot \Big[ \frac{1}{z}   - ( 1-\lambda) \Big]
=    ( 2\rho-1).
\EEQ
{Although \eq{nesterov-22} admits an elementary integral-based solution, the resulting expression is cumbersome and not convenient for the subsequent coefficientwise-product calculation. The following lemma instead yields a simple expression for the derivative needed below} (see proof in Appendix~\ref{app:lemma:nest-2}).

\begin{lemma}
\label{lemma:nest-2}
If $B$ satisfies \eq{nesterov-22} {and $\rho \geqslant 1$}, then $D(z,\lambda) = z^{2\rho-2} ( B(z,\lambda) - 1 ) $ satisfies
$$
D^{(2\rho-1)}(z,\lambda)  = \frac{ (2\rho-{1})!}{ ((1-z)^2(1-\lambda) + \lambda)^\rho}
.$$
\end{lemma}

Our goal is thus to compute an equivalent of
$$
C(z) = \frac{ 1}{ (  (1-z)^2  (1-\lambda)  + {\lambda})^\rho}
* \frac{ 1}{ (  (1-z)^2  (1-\lambda)  + {\lambda})^\rho},
$$
and then multiply the obtained equivalent of the associated sequence $(c_k)_{k \geqslant 0}$ by $[ (2\rho-1)!]^2 k^{2 - 4\rho}$. In order to obtain a formula for $C$, we rely on the following identities:
\BEAS
\frac{ 1}{ (  (1-z)^2 (1-\lambda) + {\lambda})^\rho}
& = &\frac{(-1)^{\rho -1}}{(\rho-1)!}   \frac{\partial^{\rho-1}}{\partial \alpha^{\rho-1}} \Big( \frac{ 1}{      (1-z)^2 (1-\lambda)  + {\lambda + \alpha}} \Big) \Big|_{\alpha =0}
\EEAS
leading to, using the bilinearity of convolutions,
\BEAS
& & \frac{ 1}{ (  (1-z)^2  (1-\lambda) + {\lambda})^\rho}
* \frac{ 1}{ (  (1-z)^2  (1-\lambda) + {\lambda})^\rho} \\
& = & \frac{1}{(\rho-1)!^2} 
\frac{\partial^{2\rho-2}}{\partial \alpha^{\rho-1}\partial \beta^{\rho-1}} \Big( \frac{ 1}{     (1-z)^2  (1-\lambda) + {\lambda} + \alpha} *  \frac{ 1}{      (1-z)^2 (1-\lambda)  + {\lambda} + \beta }\Big) \Big|_{\alpha,\beta =0}.
\EEAS
Moreover, using explicit calculations for rational functions described in Appendix~\ref{app:conv},
this allows us to get a rational function in $z$ and $\lambda$. We can then perform asymptotic calculations of the $z$-transform around $z=1$. This leads to the following conjecture.

\begin{conjecture}[Nesterov acceleration, $\rho \in \mathbb{N}$]
\label{conj:nesterov-2}
Assume \textbf{(A1)} obtained from Lemma~\ref{lemma:GD}. Then for the sequence defined in \eq{nesterov-2} and in \eq{perf-nesterov-2}, we have:
$$
a_k \sim c \frac{\Gamma(2 \rho)^2}{\Gamma(\rho)^2} \cdot \left\{
\begin{array}{ll}
  \ds \frac{\Gamma(\rho-\omega) \Gamma(\omega)}{\Gamma(4\rho-1-2\omega)}  
 \frac{\Gamma(2\rho-1/2-\omega)}{\Gamma(\rho+1/2-\omega)}
 \frac{2^{2\rho-1}}{4^\omega }  \cdot \frac{1}{k^{ 2\omega}} & \mbox{ if } \omega \in (0,\rho)\\[.3cm]
 \ds  \frac{1}{2^{2\rho-1}} \frac{ \log k}{k^{2\rho}} & \mbox{ if } \omega = \rho
\\[.3cm]
\ds   \frac{1}{2^{2\rho-1}}   \Gamma(\omega-\rho) \frac{1}{k^{\omega+\rho}}  & \mbox{ if } \omega > \rho.
\end{array}
\right.
$$
\end{conjecture}
As a partial proof, we do not provide a proof of the Tauberian conditions and only focus on the asymptotic expansion of the $z$-transform.
With Mathematica\footnote{See Mathematica notebook {\url{https://www.di.ens.fr/~fbach/ztf/nesterov_rho.nb}}.}, we checked the conjecture with symbolic computations for all integers $\rho$ less than $8$ (where the $z$-transform $C(z)$ already has 45 terms), for $\omega = \rho$ (logarithmic behavior), for $\omega \in (\rho-1,\rho)$, and $\omega \in (\rho,\rho+1)$ to check the three behaviors. We also conjecture (and checked empirically by running simulations) that Conjecture~\ref{conj:nesterov-2} is true for all $\rho>0$ (not necessarily integer). See Figure~\ref{fig:nestrho}.

\begin{figure}
\begin{center}
\includegraphics[width=13cm]{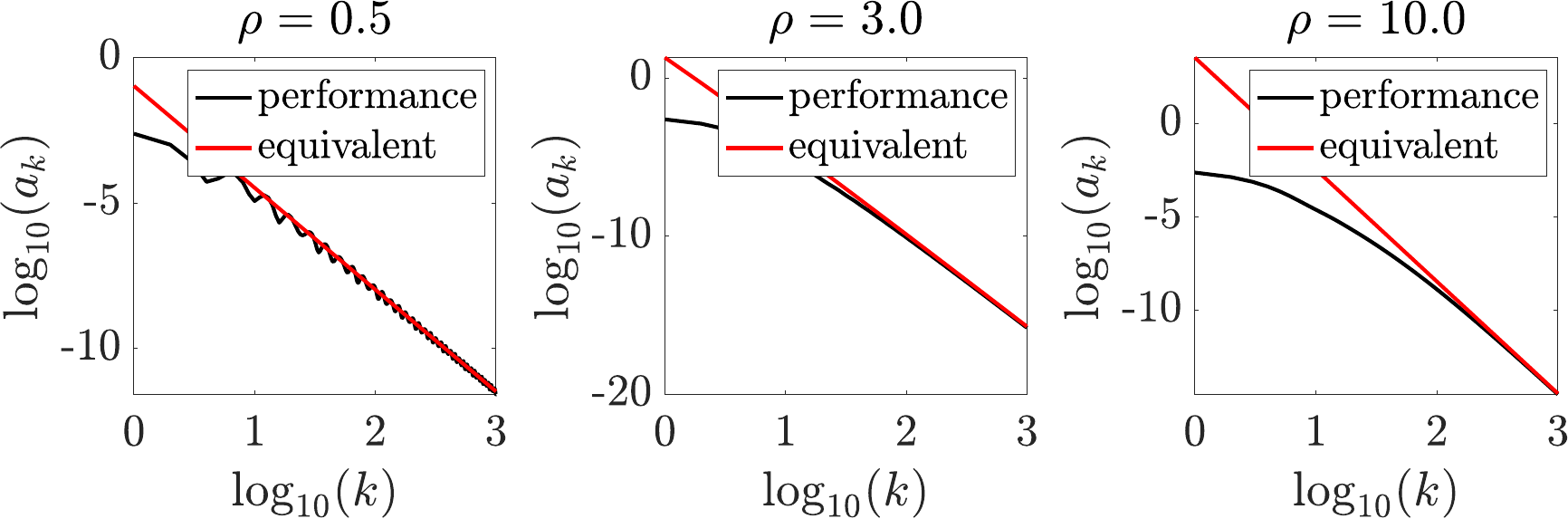}
\end{center}

\vspace*{-.5cm}

\caption{Nesterov acceleration with multiple value of $\rho$ with $\omega=3$: true performance $a_k$ vs. potential asymptotic equivalent  $\bar{a}_k$.
\label{fig:nestrho}}
\end{figure}

\paragraph{Heavy-ball recursion.}
The corresponding heavy-ball recursion for general $\rho$ can be taken to be 
\BEAS \theta_{k+1}  &  = &    \Big( \idm - \frac{k}{k+2\rho-1} \gamma H\Big)  \theta_{k} + \Big( 1 - \frac{ 2\rho  }{k+2\rho-1} \Big) ( \theta_{k} - \theta_{k-1}),
 \EEAS
 with a recursion that is asymptotically equivalent to that of~\cite{berthier2020accelerated}, as outlined by~\citet{berthier2022acceleration}, who show, in the context of gossip algorithms, an asymptotic equivalence with a partial differential equation. 
 Using similar derivations as for Nesterov acceleration, we can get an equation for $\Theta$ as 
 $$
 c(z) \Theta'(z) + 2(\rho-1) \Big[ \frac{c(z)}{z}  - \frac{c'(z)}{2} \Big] \cdot \big[ \Theta(z) - \theta_0 \big]
 = (2\rho-1) \theta_0,
 $$
 with $c(z) = \idm - (2\idm - \gamma H)z + z^2 \idm$, instead of $c(z) = \idm - (2z - z^2)(\idm - \gamma H)$.
Similar developments could then be carried out, with a similar potential lack of Tauberian conditions (and not full equivalence).

\section{Additive noise}
\label{sec:additivenoise}
 The $z$-transform method can also be used efficiently in the presence of additive noise in the recursion. We consider in this section two examples: gradient descent and Nesterov acceleration from \mysec{nesterov-1}.
 
 \subsection{Gradient descent}
 \label{sec:sgd-noise}
 The recursion in \eq{GDrec} now becomes
 $$
 \theta_k = ( \idm - \gamma H) \theta_{k-1} + \varepsilon_{k},
 $$
 with $\varepsilon_k$ a random vector with zero mean and covariance matrix $\Sigma$. For simplicity, we assume that all noise vectors are independent with the same covariance matrix, but the same technique could be applied with a time-varying covariance matrix, leading to ordinary differential equations (if the time variation is rational).
 
 Unrolling the recursion leads to 
 $$
\theta_k = ( \idm - \gamma H)^k \theta_0 + \sum_{i=1}^k ( \idm - \gamma H)^{k-i} \varepsilon_i,
$$
with the performance being in expectation equal to
$$
\frac{1}{2}\E \big[ \langle \theta_k,  H \theta_k \rangle \big]
=  \frac{1}{2}\langle \theta_0, ( \idm - \gamma H)^{2k} H \theta_0 \rangle + \frac{1}{2}
\sum_{i=1}^k \tr \big( H( \idm - \gamma H)^{2k-2i} \E [ \varepsilon_i\varepsilon_i^\top] \big).
$$
With a constant covariance matrix $\Sigma$, we get 
\BEQ
\label{eq:GDnoise}
a_k = \frac{1}{2} \E \big[ \langle \theta_k,  H \theta_k \rangle \big]
=  \frac{1}{2}\langle \theta_0, ( \idm - \gamma H)^{2k} H \theta_0 \rangle +  \frac{1}{2}
\sum_{i=0}^{k-1}  \tr \big( \Sigma H( \idm - \gamma H)^{2i}  \big).
\EEQ
We can compute the $z$-transform of $(a_k)_{k \geqslant 0}$ as follows:
\BEAS
\sum_{k=0}^{+\infty}
a_k z^k & = & \frac{1}{2}\sum_{k=0}^{+\infty} \langle \theta_0, ( \idm - \gamma H)^{2k} H  \theta_0 \rangle z^k
+  \frac{1}{2}\sum_{k=1}^{+\infty}
\sum_{i=0}^{k-1} z^k  \tr \big( \Sigma H( \idm - \gamma H)^{2i}  \big)\\
& = &  \frac{1}{2}\sum_{k=0}^{+\infty} \langle \theta_0, ( \idm - \gamma H)^{2k} H \theta_0\rangle z^k
+ \frac{1}{2}\sum_{i=0}^{+\infty}   \sum_{k=i+1}^{+\infty}
z^k  \tr \big( \Sigma H( \idm - \gamma H)^{2i}  \big)
\\
& = & \frac{1}{2} \sum_{k=0}^{+\infty}\langle \theta_0, ( \idm - \gamma H)^{2k} H  \theta_0 \rangle z^k
+  \frac{1}{2}\frac{z  }{1-z}
 \sum_{i=0}^{+\infty}  z^i  \tr \big( \Sigma H( \idm - \gamma H)^{2i}  \big).
\EEAS
This can be decomposed into the eigenvectors $u_i \in \H$ of $H$, that is,
\BEAS
A(z)  & = & \frac{1}{2} \sum_{i={1}}^{+\infty} \sum_{k=0}^{+\infty} ( 1 -  \lambda_i )^{2k} \frac{\lambda_i}{{\gamma}}   z^k
\big[ \langle \theta_0, u_i \rangle^2   + \frac{z  }{1-z} \langle u_i, \Sigma u_i \rangle \big].
\EEAS
Classically in stochastic quadratic optimization, the performance decomposes into two terms, one that corresponds to the deterministic recursion, which we will call the ``bias'' term, and one that corresponds to the noise, which we will call the ``variance'' term \citep[see][and references therein]{dieuleveut2017harder}.

\paragraph{Bias term.} This is exactly the term from \mysec{GD} with an equivalent obtained in Proposition~\ref{prop:gd}, and is proportional to $1/k^{\omega}$ where $\omega > 0$, with constants obtained from Lemma~\ref{lemma:GD}.

\paragraph{Variance term.} For this term, we make an assumption on the noise covariance matrix $\Sigma$ of a similar form as in Lemma~\ref{lemma:GD}, to obtain the variance term (together with assumptions on $H$).

\begin{proposition}[Variance term for noisy gradient descent]
Assume $\lambda_i = \frac{\gamma L}{i^\alpha}$ for $\alpha>1$, and $\langle u_i, \Sigma u_i \rangle = \frac{\varsigma^2}{i^{\beta'}}$ 
where $(\lambda_i,u_i)_{i \geqslant 1}$ are eigenvalues and eigenvectors of ${\gamma} H$. When $\theta_0=0$, the sequence $(a_k)_{k \geqslant 0}$ defined in \eq{GDnoise}, satisfies, with $\omega' = \frac{\beta'-1}{\alpha} + 1$ and $c' = \frac{L \varsigma^2}{2 \alpha (\gamma L)^{\omega'}}$:
$$
a_k \sim  \left\{
\begin{array}{ll}
  \ds \frac{c'}{2^{\omega'}}  \frac{ \Gamma(\omega')  \Gamma(1-\omega')}{ \Gamma(2-\omega')}  {k^{1-\omega'}} & \mbox{ if } \omega' \in (0,1), \\[.3cm]
  \ds \frac{1}{{2}\gamma}\tr  \big( \Sigma   ( 2   - \gamma H )^{-1} \big)  = \frac{\varsigma^2}{{2}\gamma} \sum_{i\geqslant 1} 
  \frac{1}{2  - \gamma \lambda_i}
  & \mbox{ if } \omega'>1 .
 \end{array}
\right.
$$
\end{proposition}
\begin{proof}
We can apply Lemma~\ref{lemma:discrete}, and get the constants $\omega'$ and $c'$, and then the desired equivalents by using Assumption \textbf{(A1)}. When $\omega' \in (0,1)$, we obtain a $z$-transform proportional to $(1-z)^{\omega'-2}$ and thus the equivalent proportional to $ \frac{1}{k^{\omega'-1}}$. When $\omega' >1$, then we obtain a $z$-transform proportional to $1/(1-z)$, and a limit for $a_k$.
\end{proof}

Two possible values of $\beta'$ correspond to two classical cases in optimization:
\BIT
\item Isotropic noise ($\beta'=0$): this leads to an equivalent proportional to $k^{1/\alpha}$. We thus see that isotropic noise will lead to a blow-up of the performance in infinite dimensions; note that this does not contradict non-asymptotic bounds~\citep[Lemma 2]{bach2013non}, which typically have the term $\tr(\Sigma)$.
 \item Least-squares regression ($\beta'={\alpha})$: this corresponds to the noise covariance matrix $\Sigma$ being proportional to $H$ \citep[see][for an in-depth discussion]{dieuleveut2017harder}. Then $\omega' =  2 - 1/\alpha \in (1,2)$ if $\alpha > 1$ (as is standard when $H$ has finite trace). Then, the variance term tends to a constant. 
\EIT

\subsection{Nesterov acceleration}
We consider the Nesterov-type iteration, as done by \citet{flammarion2015averaging}, with extra additive noise, leading to the following recursion instead of \eq{nesterov-1}:
\BEA
\label{eq:nesterov-noisy} \theta_{k+1}  &  = &  ( \idm - \gamma H) \Big[  \theta_{k} + \Big( 1 - \frac{ 2  }{k+1} \Big) ( \theta_{k} - \theta_{k-1}) \Big]  + \varepsilon_{k+1} ,
\EEA
initialized with $\theta_1 = \theta_0$, with i.i.d.~noise with zero mean and covariance matrix $\Sigma$.

Following \citet{flammarion2015averaging}, we extend the computations from \mysec{nesterov-1}: the equivalent iteration for  $\xi_k = k \theta_{k}$ is 
$$
\xi_{k+1} = ( \idm - \gamma H) ( 2\xi_k - \xi_{k-1}) + (k+1) \varepsilon_{k+1},
$$ 
with $\xi_0 = 0$, $\xi_1 = \theta_0$, and now a noise variable which is amplified by the factor $k+1$. Using the same reasoning as in \mysec{nesterov-1}, we get
$$
\Xi(z) - z\theta_0 =  (2z - z^2) ( \idm - \gamma H) 
\Xi(z) + \sum_{k=1}^{+\infty} z^{k+1} (k+1) \varepsilon_{k+1},$$
leading to
$$
  \Xi(z) = z \big[ \idm - (2z - z^2) ( \idm - \gamma H) \big]^{-1} \Big[
 \theta_0 + \sum_{k=1}^{+\infty} z^{k} (k+1) \varepsilon_{k+1}
\Big]
= z c(z)^{-1} \Big[
\theta_0 + \sum_{k=1}^{+\infty} z^{k} (k+1) \varepsilon_{k+1}
\Big]
,
$$
with $c(z) =  \idm - (2z - z^2) ( \idm - \gamma H) $.
We thus get, for any eigenvector $u_i$ of $H$, and $c_i(z) = 1 - (2z - z^2)(1 - \gamma h_i)$,
\BEAS
&& \E \Big[
  \langle  \Xi(z), u_i\rangle  *   \langle \Xi(z), u_i \rangle
\Big] \\
& = & \langle u_i ,\theta_0 \rangle^2  \big( \frac{z}{c_i(z)} *  \frac{z}{c_i(z)}  \big)
+ z \sum_{k=1}^{+\infty}  (k+1) ^2 \E[ \langle \varepsilon_{k+1}, u_i \rangle^2\big] \big(  \frac{z^k}{c_i(z)}  *  \frac{z^k}{c_i(z)}  \big)
\\
& = & \langle u_i ,\theta_0 \rangle^2   \big(  \frac{z}{c_i(z)}  *   \frac{z}{c_i(z)}  \big)
+ \sum_{k=1}^{+\infty}  (k+1) ^2 z^k \E[ \langle \varepsilon_{k+1}, u_i \rangle^2\big]   \big(  \frac{z}{c_i(z)}  *   \frac{z}{c_i(z)} \big) \\
& = & \Big[
 \langle u_i ,\theta_0 \rangle^2 + \langle u_i, \Sigma u_i \rangle 
 \sum_{k=1}^{+\infty}  (k+1) ^2 z^k 
\Big] \big( \frac{z}{c_i(z)}  *   \frac{z}{c_i(z)}  \big).
\EEAS
Here, we have used that
$ ( z^k A(z)) * ( z^k B(z)) =   z^k \cdot (A*B)(z)$. The final $z$-transform of the performance of $\xi_{k+1}$ can then be derived with two terms. As for gradient descent, we have a bias term (which corresponds to the absence of noise, that is, Proposition~\ref{prop:nesterov-1}) and a variance term corresponding to the $z$-transform  
$$
 \Big[\langle u_i, \Sigma u_i \rangle 
 \sum_{k=1}^{+\infty}  (k+1) ^2 z^k 
\Big] \big(  \frac{z}{c_i(z)}  *   \frac{z}{c_i(z)} ) \sim   
 \frac{2}{(1-z)^3} \langle u_i, \Sigma u_i \rangle 
 \big( \frac{z}{c_i(z)}  *   \frac{z}{c_i(z)}  \big).
$$
This leads to the following proposition.
\begin{proposition}[Variance term for noisy Nesterov acceleration]
Assume $\lambda_i = \frac{\gamma L}{i^\alpha}$ for $\alpha>{1}$, and $\langle u_i, \Sigma u_i \rangle = \frac{\varsigma^2}{i^{\beta'}}$ 
where $(\lambda_i,u_i)_{i \geqslant 1}$ are eigenvalues and eigenvectors of ${\gamma} H$. When $\theta_0=0$, for the sequence $(\theta_k)_{k \geqslant 0}$ defined in \eq{nesterov-noisy}, we have, with $\omega' = \frac{\beta'-1}{\alpha} + 1$ and $c' = \frac{L \varsigma^2}{2 \alpha (\gamma L)^{\omega'}}$:
$$
\frac{1}{2} \langle \theta_k, H \theta_k \rangle \sim c' \cdot  \left\{
\begin{array}{ll}
  \ds \frac{\Gamma(1-\omega') \Gamma(\omega')}{\Gamma(6-2\omega')} {4^{1-\omega'}}    k^{ 3-2\omega'} 
 & \mbox{ if } \omega' \in (0,1), \\[.3cm]
  \ds \frac{\Gamma(2-\omega')}{\Gamma(5-\omega')}\frac{\Gamma(\omega')}{\omega'-1} k^{ 2-\omega'} 
  & \mbox{ if } \omega' \in (1,2) .
 \end{array}
\right.
$$
\end{proposition}
\begin{proof}
 When $\omega' \in (0,1)$, we obtain from  \eq{proof-nest1} in the proof of Proposition~\ref{prop:nesterov-1}, that $\sum_{i \geqslant 1} \langle u_i, \Sigma u_i \rangle 
 \big( \frac{z}{c_i(z)} *   \frac{z}{c_i(z)} \big)$ has an equivalent proportional to $(1-z)^{2\omega'-3}$, leading to the impact on the noise of $\frac{1}{k^2} \times k^3 \times k^{ 3-2\omega'- 1} \sim  k^{ 3-2\omega'} $. 
 
 When $\omega' >1 $, we obtain from  \eq{proof-nest2} in  the proof of Proposition~\ref{prop:nesterov-1}, that $\sum_{i=1}^{+\infty} \langle u_i, \Sigma u_i \rangle 
 \big( \frac{z}{c_i(z)}  *  \frac{z}{c_i(z)}  \big)$ has an equivalent proportional to  $(1-z)^{\omega'-2} $, leading to the impact on the noise of $\frac{1}{k^2} \times k^3 \times k^{2- \omega' -1 } \sim k^{ 2-\omega'}$.
\end{proof}

Using the same cases as in \mysec{sgd-noise}, we get asymptotic versions of results from \citet{flammarion2015averaging}; see experiments in Fig.~\ref{fig:nesterovnoise}:
\BIT
\item Isotropic noise ($\beta'=0$): this corresponds to $\omega' = 1-1/\alpha \in (0,1)$. We then obtain an impact on the noise proportional to $k^{ 1 + 2/\alpha}$ (so diverging even more than gradient descent).

\item Noise corresponding to least-squares regression ($\beta'=1$): this corresponds to $\omega' = 2-1/\alpha \in (1,2)$. We then obtained the equivalent $k^{ 1/\alpha}$, which is still exploding, but more slowly. \EIT

 \begin{figure}
\begin{center}
\includegraphics[width=8.4cm]{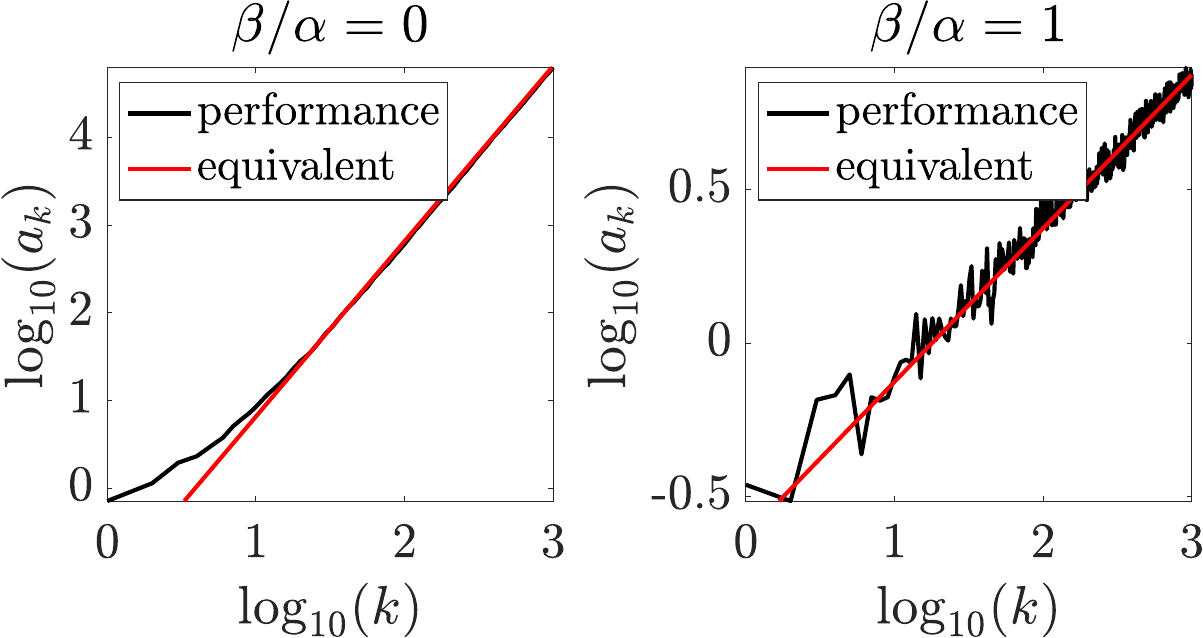}
\end{center}

\vspace*{-.5cm}

\caption{Variance term for Nesterov acceleration with additive noise: $\beta = 0$ (left) and $\beta=\alpha$ (right). 
\label{fig:nesterovnoise}}
\end{figure}

\section{Averaging}
\label{sec:averaging}
In particular, in stochastic settings, various forms of averaging are used~\citep{polyak1992acceleration,dieuleveut2017harder}. The $z$-transform again provides an algebraically simple way to see its effect. Given the $z$-transform $\Theta$ of some vector-valued sequence, we have, for the $z$-transform
$\bar\Theta$ of the sequence $(\bar{\theta}_k)_{k \geqslant 0}$ with $\bar\theta_k = \frac{1}{k+1} \sum_{i=0}^{k} \theta_i$;
\BEAS
\bar\Theta(z) = \sum_{k=0}^{+\infty} z^k \frac{1}{k+1} \sum_{i=0}^k \theta_i
& = & \sum_{i=0}^{+\infty} \theta_i \sum_{k=i}^{+\infty} \frac{z^{k}}{k+1},
\EEAS
leading to 
$$
\frac{d}{dz} \big[ z \bar\Theta(z) \big] =  \sum_{i=0}^{+\infty} \theta_i \sum_{k=i}^{+\infty}z^k = 	 \frac{\Theta(z)}{1-z}.
$$
Calling ${\rm D}$ the operator defined as ${\rm D}A(z) = \frac{d}{dz} [ z A(z) ]$, this is simply the $z$-transform of $( (k+1)a_k)_{k \geqslant 0}$. Thus, since ${\rm D}(A * B) = A * {\rm D } B  =  {\rm D} A * B$. We can thus consider, for $\Theta(z)$ the $z$-transform of the gradient descent sequence in \eq{GDrec},
\BEA
\label{eq:AGDP}
\!\!\!\!\!\!\!\!\!\!\!\!\frac{ \langle \Theta(z), u_i \rangle}{1-z} * \frac{ \langle \Theta(z), u_i \rangle}{1-z}
& = &\langle \theta_0, u_i \rangle^2 \big(\frac{1 }{1-z} \frac{1}{1 - z( 1- \gamma h_i)} \big)
* \big(\frac{1 }{1-z} \frac{1}{1 - z( 1- \gamma h_i)} \big) ,
\EEA
and obtain equivalents with the same technique as for Nesterov acceleration, which is not surprising given their similarity~\citep{flammarion2015averaging}; see proof in Appendix~\ref{app:AVGD}.
\begin{proposition}[Averaged gradient descent]
\label{prop:AVGD}
Assume \textbf{(A1)} obtained from Lemma~\ref{lemma:GD}. For averaged gradient descent, we have
$$
\frac{1}{2} \langle \bar\theta_k, H \bar\theta_k \rangle \sim c \cdot  \left\{
\begin{array}{ll}
  \ds   2 c \frac{  \Gamma(\omega)^2} {\Gamma(3-\omega)}
\frac{
 2^{1-\omega} -1
}{1-\omega} \frac{1}{k^\omega}
 & \mbox{ if } \omega \in (0,2), \\[.3cm]
  \ds   \frac{1}{2{\gamma^2} k^2} \langle \theta_0  , H^{-1}  \theta_0  \rangle & \mbox{ if } \omega >2 .
 \end{array}
\right.
$$
\end{proposition}

\paragraph{Extensions.} We could also consider a mix of averaging and acceleration, as proposed by~\citet{flammarion2015averaging}, as well as non-uniform averaging such as $\bar \theta_k = \frac{ \sum_{i=0}^k i  \theta_i}{ k(k+1)/2 }.
$

   \section{Stochastic gradient descent}
   \label{sec:sgd}
In \mysec{additivenoise}, we considered the gradient descent recursion with additive noise. This is not, however, applicable to the type of stochasticity in machine learning. We consider instead the ``least-mean-squares'' (LMS) recursion~\citep{bershad1986analysis,feuer2003convergence,Mac95,bach2013non,defossez2015averaged,dieuleveut2017harder,berthier2020tight,varre2021last}:
$$
\eta_k = \eta_{k-1} - \gamma x_k\big[ \langle x_k, \eta_{k-1} \rangle - y_k \big],
$$
where $(x_k,y_k) \in \H \times \rb$ is sampled i.i.d. This is exactly single-pass stochastic gradient descent on a quadratic cost. The results in this section are special cases of the work of~\citet{velikanov} with a similar use of the $z$-transform (they consider the more general case of heavy-ball with constant stepsize, and thus with a more complex proof; moreover, they go through an equivalent for the cumulative sums rather than directly on the original sequence).

\subsection{Bias-variance decomposition}
We consider the model $y_k = \langle x_k , \eta_\ast \rangle + \varepsilon_k$, with $\varepsilon_k$ independent of $x_k$ with zero mean and variance $\varsigma^2$,  thus leading to
$$
\eta_k - \eta_\ast = \big[ \idm - \gamma x_k \otimes x_k \big](\eta_{k-1} - \eta_\ast )  +\gamma \varepsilon_k x_k.
$$
With $H = \E [ x \otimes x]$ an operator from $\H$ to $\H$, we can write the recursion as:
$$
\eta_k - \eta_\ast = \big[ \idm - \gamma H  \big](\eta_{k-1} - \eta_\ast )  
+  \gamma \big[ H -  x_k \otimes x_k \big](\eta_{k-1} - \eta_\ast ) +\gamma \varepsilon_k x_k,
$$
and we can see it as a noisy gradient descent recursion, but the difference with \mysec{additivenoise} is the presence of the \emph{multiplicative} noise $\gamma \big[ H -  x_k \otimes x_k \big](\eta_{k-1} - \eta_\ast ) $ on top of the \emph{additive} noise $\gamma \varepsilon_k x_k$ \citep[see][for further discussions on noise types]{dieuleveut2017harder}.

It is classical to separate the iterate $\theta_k = \eta_k - \eta_\ast$ into two terms, each satisfying its own recursion
\BEAS
\theta_k^{\rm bias} & = &  \big[ \idm - \gamma x_k \otimes x_k \big]\theta_{k-1}^{\rm bias} 
\hspace*{2.02cm} \mbox{ with } \theta_0^{\rm bias} = \theta_0
\\
\theta_k^{\rm var} & = &  \big[ \idm - \gamma x_k \otimes x_k \big]\theta_{k-1}^{\rm var}  +\gamma \varepsilon_k x_k
\hspace*{0.58cm} \mbox{ with } \theta_0^{\rm var} = 0.
\EEAS

  Denoting $\Theta_k = \theta_k \otimes \theta_k =  (\eta_k -\eta_\ast) \otimes (\eta_k -\eta_\ast) $ an operator from $\H$ to $\H$, we have:
  $$
  \E \big[ \Theta_k ] 
  = \big( \idm - \gamma H \otimes \idm - \idm \otimes \gamma H
  + \gamma^2 \E [ x_k \otimes x_k \otimes x_k \otimes x_k] \big) \E \big[ \Theta_{k-1}\big]
  + \gamma^2 \varsigma^2 H.
  $$
  Following~\citet{bershad1986analysis,feuer2003convergence,defossez2015averaged}, we denote $T$ the operator on self-adjoint operators equal to (when applied to self-adjoint operators)
  \BEQ
  \label{eq:T}
  T = H \otimes \idm + \idm \otimes H - \gamma \E [ x  \otimes x  \otimes x  \otimes x ],
  \EEQ
  leading to 
  $$
  \E \big[ \Theta_k ] 
  = \big( \idm - \gamma T \big) \E \big[ \Theta_{k-1}\big]
  + \gamma^2 \sigma^2 H.
$$
Thus, we can compute in closed form:
$$
\E \big[ \Theta_k ]  = 
  \big( \idm - \gamma T \big)^k \Theta_0 + \gamma  \varsigma^2 \big[ \idm -  \big( \idm - \gamma T \big)^k \big] T^{-1} H,
   $$
   with the usual criterion $a_k = \frac{1}{2} \langle \E \big[ \Theta_k ] ,H\rangle$, where $\langle M,N \rangle = \tr (MN)$ denotes the usual dot-product between self-adjoint operators. We have two terms:
   \BEA
   \label{eq:bias-lms}
\E \big[ \Theta_k^{\rm bias} ]  & = & 
  \big( \idm - \gamma T \big)^k \Theta_0  \\
   \label{eq:var-lms} \E \big[ \Theta_k^{\rm var} ]  & = &   \gamma  \varsigma^2 \big[ \idm -  \big( \idm - \gamma T \big)^k \big] T^{-1} H,
   \EEA
that can be studied separately.

\paragraph{Maximal stepsize.}
 Following~\citet{defossez2015averaged}, we define $\gamma_{\max}$ as the largest $\gamma>0$ such that the operator $T$ is positive semi-definite, that is,
 \BEQ
 \label{eq:Z}
 \frac{2}{\gamma_{\max}} = \sup_{A: \H \to \H} \frac{ \E [ \langle x, A x\rangle^2]}{\tr[ A^\ast H A]  }
 \EEQ
 (with $A$ reduced to self-adjoint operators).
 We then know from~\citet{defossez2015averaged} that $\gamma \in [0,\gamma_{\max}]$ implies that
 $-\idm \preccurlyeq \idm - \gamma T \preccurlyeq \idm$. Moreover, we have $\gamma_{\max} \leqslant \frac{2}{ \tr[H]}$ and $\E [ \langle x,x\rangle x \otimes x ] \leqslant \frac{2}{\gamma_{\max}} H$, which implies that if 
 $\| x\|^2 $ is constant almost surely, we have $\gamma_{\max} = \frac{2}{ \tr[H]}$ \citep[see][Section 2.1]{defossez2015averaged}.

\subsection{Simplified model}

 In order to obtain asymptotic equivalents, we consider the following model for the input $x$ 
 \BEQ
 \label{eq:modelLMS}
x= \sum_{i = 1}^{+\infty}  h_i^{1/2} z_i u_i,
 \EEQ
 where $H = \sum_{i=1}^{+\infty} h_i u_i \otimes u_i$ is an eigendecomposition of $H$, and the random variables $(z_i)_{i \geqslant 1}$ are \emph{independent} and satisfy:
 $$
 \E[ z_i z_j] = 1_{i=j}, \ \E[ z_i ] = 0, \mbox{ and } \E[z_i^4] = 3+\kappa,
 $$
 for all $i,j \geqslant 1$. This implies $\E [ z \otimes z] = \idm$ and $\E[ x \otimes x] = H$, since from \eq{modelLMS} we have $x = H^{1/2} U z$, where $U$ is the matrix of eigenvectors of $H$. The constant $\kappa$ is the common marginal excess kurtosis of the variables $z_i$, $i \geqslant 1$ (equal to $0$ for a Gaussian distribution).
 
 This is a classical ``high-dimensional'' setup \citep[see, e.g.,][and references therein]{bach2024high}, that includes~$x$ Gaussian with mean zero and covariance $H$ (which corresponds to $\kappa=0$), which has been already studied in this context for the last 40 years~\citep{bershad1986analysis,feuer2003convergence,paquette20244,meterez2025simplified}, and $z$ a vector of Rademacher random variables (which corresponds to $\kappa=-2$, and for which we have $\|x\|^2 = \tr(H)$ almost surely). Note that we always have $\kappa \geqslant -2$. An alternative, simpler model could also be considered~\citep{slock1993convergence}, and more general models have been considered to incorporate positive-definite kernel methods~\citep[][Section 3]{velikanov2023a}.
 
  We now consider, for $i\leqslant j$, the infinite-dimensional basis $v_{ij}= \frac{1}{2} u_i \otimes u_j + \frac{1}{2} u_j \otimes u_i$ of the set of self-adjoint operators on $\H$. The following lemma is standard in the analysis of the LMS algorithm with Gaussian inputs; see \citet[Eqs.~(14-17)]{feuer2003convergence} and~\citet[Section III]{horowitz1981performance}. It is here shown for the slightly more general model from \eq{modelLMS}.
    \begin{lemma}
   For the model defined in \eq{modelLMS}, and $T$ defined in \eq{T}, we have, for any $i \leqslant j$,
    $$
 T v_{ij} = \big[ h_i + h_j -  2 \gamma h_i h_j  -     \gamma  \kappa 1_{i=j} h_i^2 ] v_{ij} - \gamma 
 1_{i=j} h_i \sum_{k\geqslant 1} h_k v_{kk} .
 $$
   \end{lemma}
\begin{proof}  
We compute for $k \leqslant \ell$ and $i \leqslant j$, $\langle v_{k \ell}, T v_{ij}\rangle$, first when $i=j$, and then $i<j$. We have for $i=j$,
\BEAS
\langle v_{k \ell}, \E [ x  \otimes x  \otimes x  \otimes x ] v_{ii}\rangle & = & \E[ \langle x, u_i \rangle^2 \langle x, u_k \rangle
 \langle x, u_\ell \rangle],
\EEAS
which is equal to $h_i^2\E [ z_i^4] = h_i^2 ( \kappa + 3) $ if $k= \ell = i$, to $h_i h_k$ if $k=\ell \neq i$, and to zero otherwise. This leads to
$$
T v_{ii} = 2 h_i v_{ii} - \gamma h_i \sum_{ k\geqslant 1} h_k v_{kk} - \gamma  (\kappa+2) \ h_i^2 v_{ii}
.$$ Moreover, for $i< j$, we have
$
\langle v_{k \ell}, \E [ x  \otimes x  \otimes x  \otimes x ] v_{ij}\rangle =  \E[ \langle x, u_i \rangle \langle x, u_j \rangle \langle x, u_k \rangle
 \langle x, u_\ell \rangle]$, which
is equal to zero if $(k,\ell) \neq (i,j)$, and $h_i h_j$ otherwise; thus, since $\langle v_{ij},v_{ij}\rangle = \frac{1}{2}$,
$
T v_{ij} = [ h_i + h_j ] v_{ij} - 2 \gamma h_{i} h_j v_{ij}.$ These two cases can be combined into the desired result.
\end{proof}
The previous lemma shows that the operator $T$ is block-diagonal, with one-dimensional blocks for each $i<j$, with eigenvalue 
 $h_i + h_j - 2 \gamma h_i h_j$, and a block for all vectors $v_{kk}$, $k \geqslant 1$, equal to
 $$
V = 2 \Diag(h) -  \gamma  (\kappa+2) \Diag(h)^2 - \gamma h \otimes h
= \Diag( 2 h - \gamma(\kappa+2) h \circ h)  - \gamma h \otimes h .
 $$
 This implies that, in the basis of eigenvectors, off-diagonal elements of $\E \big[ \Theta_k ] $ evolve independently (and converge to zero linearly), while the diagonal elements that are needed to compute the performance measure evolve according to a linear iteration with a matrix that is the sum of a diagonal and a rank-one matrix.
 
  The maximal stepsize such that $T \succcurlyeq 0$ is then the largest $\gamma$ such that 
 $\Diag( 2 h - \gamma(\kappa+2) h \circ h)  \succcurlyeq \gamma h \otimes h$, which is equivalent, through Schur complements~\citep{horn2012matrix}, to $\gamma \langle h, \Diag( 2 h - \gamma(\kappa+2) h \circ h)^{-1} h \rangle \leqslant 1$, that is,\footnote{Note that the limiting stepsize defined in \eq{rho} was already obtained by~\citet{horowitz1981performance} using a $z$-transform technique (but no asymptotic equivalents were derived).}
  \BEQ
  \label{eq:rho}
\upsilon = \sum_{i \geqslant 1} \frac{ \gamma h_i}{2 - (\kappa+2) \gamma h_i} \leqslant 1.
 \EEQ
 If \eq{rho} is satisfied, then $-\idm \preccurlyeq \idm - \gamma T \preccurlyeq \idm$~\citep{defossez2015averaged}.
 If $\kappa = -2$ (constant feature norm), then this is simply $\gamma_{\max} = 2 / \tr[H]$ (note that we could also go through \eq{Z}). Note that $\upsilon$ can be expressed from another spectral measure only related to $H$. We introduce $d\tau(\lambda) = \sum_{i \geqslant 1} \gamma h_i {\rm Dirac}(\lambda| \gamma h_i)$, with its \sti transform $T$, that satisfies Assumption \textbf{(A1)} with $\omega = 1 - \frac{1}{\alpha}$, and $c = \frac{1}{\alpha} (\gamma L)^{1/\alpha}$, leading to the equivalents
 $T(u) \sim \frac{(\gamma L)^{1/\alpha}}{\alpha} \Gamma(\frac{1}{\alpha}) \Gamma(1-\frac{1}{\alpha}) u^{-\frac{1}{\alpha}}$ and  $T'(u) \sim - \frac{(\gamma L)^{1/\alpha}}{\alpha^2} \Gamma(\frac{1}{\alpha}) \Gamma(1-\frac{1}{\alpha}) u^{-\frac{1}{\alpha}-1}$.

 \paragraph{Computation of performance.} We can now compute an asymptotic equivalent of the performance of stochastic gradient descent. In order to study convergence, we could look at an eigenvalue decomposition of the matrix $V$ using ``secular'' equations~\citep{golub1973some}. Instead, we will use the $z$-transform method, which avoids explicit computations of eigenvalues.
 
 \begin{proposition}[Least-mean-squares]
 Assume $\langle \delta, u_i \rangle = \frac{\Delta}{i^{\beta/2}}$ and $\lambda_i = \frac{\gamma L}{i^\alpha}$, with $\omega = \frac{\beta-1}{\alpha} + 1$ {and $\alpha>1$}, $c = \frac{L \Delta^2 }{2 \alpha (\gamma L)^{\frac{\beta-1}{\alpha}+1}}$, and $\tau =   \frac{1}{2} \sum_{i\geqslant 1} \frac{  \delta_i^2}{  2 - (\kappa +2) \gamma h_i }$. Then, with the model defined in {\eq{modelLMS}},
and the performance measure $a_k^{\rm bias} = \frac{1}{2} \langle \E \big[ \Theta_k^{\rm bias} ] ,H\rangle$ and
$a_k^{\rm var} = \frac{1}{2} \langle \E \big[ \Theta_k^{\rm var} ] ,H\rangle$ defined from \eq{bias-lms} and \eq{var-lms}, we have, if $\gamma < \gamma_{\max}$ from \eq{Z},   
\BEAS
a_k^{\rm var} & \sim & \frac{1}{2} \frac{ \varsigma^2 \upsilon}{1-\upsilon} \\
 a_k^{\rm bias}  & \sim & 
 \left\{
\begin{array}{ll}
  \ds    \frac{c}{1-\upsilon} \frac{\Gamma(\omega)}{(2k)^{\omega}}
 & \mbox{ if } \omega   \in (0, 2 -\frac{1}{\alpha}) , \\[.3cm]
  \ds  \frac{\tau }{(1-\upsilon)^2}
 (2 \gamma L)^{1/\alpha} 
\Gamma\big(1- \frac{1}{\alpha}\big) \frac{1}{4} \big( 1 - \frac{1}{\alpha}\big) \frac{1}{k^{2-\frac{1}{\alpha}}} & \mbox{ if } \omega >2 -\frac{1}{\alpha} .
 \end{array}
\right.
\EEAS
 \end{proposition}
 \begin{proof}
 We consider 
 \BEAS
 a_k = \frac{1}{2} \E [ \langle \theta_k, H \theta_k \rangle]
 & = &  \frac{1}{2} \tr( H \E[ \Theta_k] )
 = \frac{1}{2} \sum_{i \geqslant 1}
 h_i \langle  \E[ \Theta_k], v_{ii} \rangle \\
 & = & 
 \frac{1}{2} \sum_{i \geqslant 1}
 h_i \langle   \big( \idm - \gamma T \big)^k \Theta_0 , v_{ii} \rangle 
 +  \frac{1}{2} \sum_{i \geqslant 1}
 h_i  \gamma  \varsigma^2 \big\langle  \big[ \idm -  \big( \idm - \gamma T \big)^k \big]    T^{-1} H, v_{ii} \big\rangle 
 . \EEAS
 With  $\delta_i = \langle \theta_0 - \theta_\ast , u_i \rangle$, we get only a contribution from the block-diagonal part (that is $V$):
 \BEAS
 a_k  & = & \frac{1}{2}  
  \langle   \delta \circ \delta,   \big( \idm - \gamma V \big)^k h  \rangle 
 +  \frac{1}{2}  \gamma  \varsigma^2 \big\langle  h, \big[ \idm -  \big( \idm - \gamma V  \big)^k \big]  V^{-1} h   \big\rangle \\
 & = & \frac{1}{2}  
  \langle   \delta \circ \delta - \gamma \varsigma^2 V^{-1} h ,   \big( \idm - \gamma V \big)^k h  \rangle 
 +  \frac{1}{2}  \gamma  \varsigma^2  \langle  h,    V^{-1} h    \rangle 
.
 \EEAS
 We can use the matrix inversion lemma~\citep{horn2012matrix}, that is,
 \BEAS
\big( D + \lambda u \otimes u \big)^{-1} u 
& = &   D^{-1/2}\big( \idm + \lambda D^{-1/2} u \otimes u D^{-1/2}  \big)^{-1} D^{-1/2}u 
=  ( 1 + \lambda \langle u, D^{-1} u \rangle)^{-1}
 D^{-1} u, \EEAS
 leading to (for the constant term):
 \BEAS
 \frac{1}{2}  \gamma  \varsigma^2  \langle  h,    U^{-1} h    \rangle & = & 
 \frac{1}{2} \varsigma^2 \langle \gamma h, \big(   \Diag(2 \gamma h -     (\kappa+2)\gamma^2 h h \circ h) - \gamma h \otimes \gamma h \big)^{-1} \gamma h  \rangle
 \\
 & = & \frac{1}{2}
  \frac{\varsigma^2 \langle \gamma h, 
  \Diag(2 \gamma h -     (\kappa+2)\gamma^2   h \circ h) ^{-1}    \gamma h \rangle  }
  { 
  1  - \langle\gamma h,
  \Diag(2 \gamma h -     (\kappa+2)\gamma^2  h \circ h) ^{-1}    \gamma h  \rangle
  } =\frac{1}{2} \frac{ \varsigma^2 \upsilon}{1-\upsilon},
 \EEAS
 where $\upsilon$ is defined in \eq{rho} (and is strictly less than $1$, 
  for $\gamma < \gamma_{\max}$). This takes care of the variance term (since the other term involving $\varsigma^2$ will be negligible).
  
 The $z$-transform of the other term is
 \BEAS
  A(z)& =  & \sum_{k=0}^{+\infty}
  z^k \frac{1}{2} \langle h, \Big( \big[  \idm -\Diag(2 \gamma h -     (\kappa+2)\gamma^2   h \circ h) \big] + \gamma h \otimes \gamma h \Big)^k  ( \delta \circ \delta - \gamma \varsigma^2 V^{-1} h) \rangle  \\
  & = &\frac{1}{2\gamma}  \Big\langle \delta \circ \delta - \gamma \varsigma^2 V^{-1} h,
  \Big(
  \idm - z  \big( \big[  \idm -\Diag(2 \gamma h -     (\kappa+2)\gamma^2   h \circ h) \big] + \gamma h \otimes \gamma h \big)
  \Big)^{-1} 
  \gamma h \Big\rangle
\\
& = &\frac{1}{2\gamma}  \Big\langle \delta \circ \delta - \gamma \varsigma^2 V^{-1} h,
  \Big(\Diag( 1 - z +   z \gamma h ( 2 - (\kappa +2) \gamma h) ) - z   \gamma h \otimes \gamma h \big)
  \Big)^{-1} 
  \gamma h  \Big\rangle
\\
& = &\frac{1}{2\gamma} \frac{ \big\langle \delta \circ \delta - \gamma \varsigma^2 V^{-1} h,
 \Diag( 1 - z +   z \gamma h ( 2 - (\kappa +2) \gamma h) )   \big)
  ^{-1} 
  \gamma h \big\rangle }{1 - z  \langle \gamma h, 
 \Diag( 1 - z +   z \gamma h ( 2 - (\kappa +2) \gamma h) )   \big)
  ^{-1} 
  \gamma h \rangle} \\
  & & \hspace*{8cm} \mbox{ by the matrix inversion lemma},
  \\
  & = & \frac{C(z) - D(z)}{1 - B(z)}.
 \EEAS
 The denominator $B(z)$ of the expression above converges to $1 - \upsilon > 0 $ when $z$ tends to one, so, to get an equivalent, we only need to study the numerator $C(z) - D(z)$.
 We have, for the bias term,
 \BEAS
  & & C(z) \\
  & \!\!\!= \!\!\!&  \frac{1}{2} \big\langle \delta \circ \delta  ,
 \Diag( 1 - z +   z \gamma h ( 2 - (\kappa +2) \gamma h) )   \big)
  ^{-1} 
  \gamma h \big\rangle  = \frac{1}{2} \sum_{i\geqslant 1} \frac{ \gamma h_i \delta_i^2}{1 - z + z \gamma h_i ( 2 - (\kappa +2) \gamma h_i )} \\
  & \!\!\!= \!\!\!& \int_0^1
\frac{  d\sigma(\lambda)}{1 - z + z \lambda ( 2 - (\kappa+2) \lambda)},
 \EEAS
 with $d\sigma$ defined in Lemma~\ref{lemma:GD}. We have
 $$
 C(1) =   \frac{1}{2} \sum_{i\geqslant 1} \frac{  \delta_i^2}{  2 - (\kappa +2) \gamma h_i }.
 $$
 When $\kappa=-2$, we exactly get the same expression as for gradient descent, and one can show that the term $\kappa+2$ has no effect in the asymptotic equivalent (except for the maximal stepsize), so we only consider the case $\kappa=-2$. 
  
 Thus, we have:
 \BEAS
 C(z) & = & \int_0^1
\frac{  d\sigma(\lambda)}{1 - z + 2 z \lambda} = \frac{1}{2z} S \Big( \frac{1-z}{2z} \Big) \\
B(z) & = & \int_0^1
\frac{  \lambda z d\tau(\lambda)}{1 - z + 2 z \lambda} = 
\int_0^1
\frac{  [\lambda z  + \frac{1-z}{2} - \frac{1-z}{2} ] d\tau(\lambda)}{1 - z + 2 z \lambda}
= \upsilon - \frac{1-z}{4z}  T \Big( \frac{1-z}{2z} \Big).
 \EEAS
 
 \paragraph{Case $\omega \in (0,1)$.} When $\omega \in (0,1)$, we have, from Lemma~\ref{lemma:GD}:
 \BEAS
 \frac{C(z)}{1-B(z)} & \sim & \frac{1}{1-\upsilon} \frac{c \Gamma(\omega) \Gamma(1-\omega) }{2^\omega} ( 1 - z)^{\omega - 1},
 \EEAS
 leading to the equivalent $ a_k \sim    \frac{c}{1-\upsilon}  \frac{\Gamma(\omega)}{(2k)^{\omega}}$, which is the same as for gradient descent in \mysec{GD}, but with a smaller stepsize $\gamma$ (which leads to larger constant $c$).
 
 \paragraph{Case $\omega \in (1,2)$.} When $\omega \in (1,2)$, from Lemma~\ref{lemma:GD}, we have
 $C'(z) \sim \frac{c \Gamma(\omega) \Gamma(2 - \omega) }{2^{\omega} }( 1 - z)^{\omega - 2}$, and thus the derivative of $C/(1-B)$ is
\BEAS
\frac{C'(z)}{1-B(z)} + \frac{C(z) B'(z) }{( 1- B(z))^2}
& \sim & \frac{1}{1-\upsilon} 
\frac{c \Gamma(\omega) \Gamma(2 - \omega) }{2^{\omega} }( 1 - z)^{\omega - 2}
 \\
 & & \hspace*{2cm} + \frac{C(1) }{(1-\upsilon)^2}
\Big[
\frac{1}{4} T\big( \frac{1-z}{2} \big)
+ \frac{1-z}{8}T'\big( \frac{1-z}{2} \big)
\Big]
\\
& \sim & \frac{1}{1-\upsilon} 
\frac{c \Gamma(\omega) \Gamma(2 - \omega) }{2^{\omega} }( 1 - z)^{\omega - 2}
+ \frac{C(1) }{(1-\upsilon)^2}
c'' (1-z)^{-\frac{1}{\alpha}},
\EEAS 
with $c'' = (2 \gamma L)^{1/\alpha} \Gamma\big( \frac{1}{\alpha}\big)
\Gamma\big(1- \frac{1}{\alpha}\big) \frac{1}{4} \big( 1 - \frac{1}{\alpha}\big)
$.
 Thus, if $\omega   < 2 -\frac{1}{\alpha}$, the term in $( 1 - z)^{\omega - 2}$ dominates and we obtain the same equivalent as for $\omega\in (0,1)$. Otherwise, the term $ (1-z)^{-\frac{1}{\alpha}}$ dominates, and we obtain a rate in $\frac{1}{k^{2-\frac{1}{\alpha}}}$, equal to
 $$
 \frac{C(1) }{(1-\upsilon)^2}
 (2 \gamma L)^{1/\alpha} 
\Gamma\big(1- \frac{1}{\alpha}\big) \frac{1}{4} \big( 1 - \frac{1}{\alpha}\big) \frac{1}{k^{2-\frac{1}{\alpha}}}.
 $$
 When $\omega   = 2 -\frac{1}{\alpha}$, we get the sum of two terms.  
 For $\omega > 2$, we get the same equivalent as $\omega > 2 -\frac{1}{\alpha}$ above.  
 Note that $\alpha > 1$ (so that we have a finite trace), so $2 -\frac{1}{\alpha} \in (1,2)$.

For the second term of the variance equivalent (corresponding to $D(z)$), the term
$
  \big\langle  \gamma \varsigma^2 V^{-1} h,
 \Diag( 1 - z +   z \gamma h ( 2 - (\kappa +2) \gamma h) )   \big)
  ^{-1} 
  \gamma h \big\rangle $
  can be shown to be negligible.
  \end{proof}
  
  These rates are illustrated in Figure~\ref{fig:lms}. We can make the following observations:
  \BIT
  \item We recover the results from~\citet{velikanov2023a} with a simpler, more direct proof, thus obtaining asymptotic results that complement the non-asymptotic results from \citet{berthier2020tight}. This is also a subcase of results from \citet{paquette20244} with a simpler proof.
  
  \item We could also use averaging, which is commonly used to improve the variance term \citep{polyak1992acceleration,bach2013non,defossez2015averaged} (this could be done using the tools of \mysec{averaging}). Moreover, the standard deviation of the performance around its mean empirically has an asymptotic equivalent that could be characterized (see a few results in Appendix~\ref{app:var}).  \EIT

 \begin{figure}
\begin{center}
\includegraphics[width=12cm]{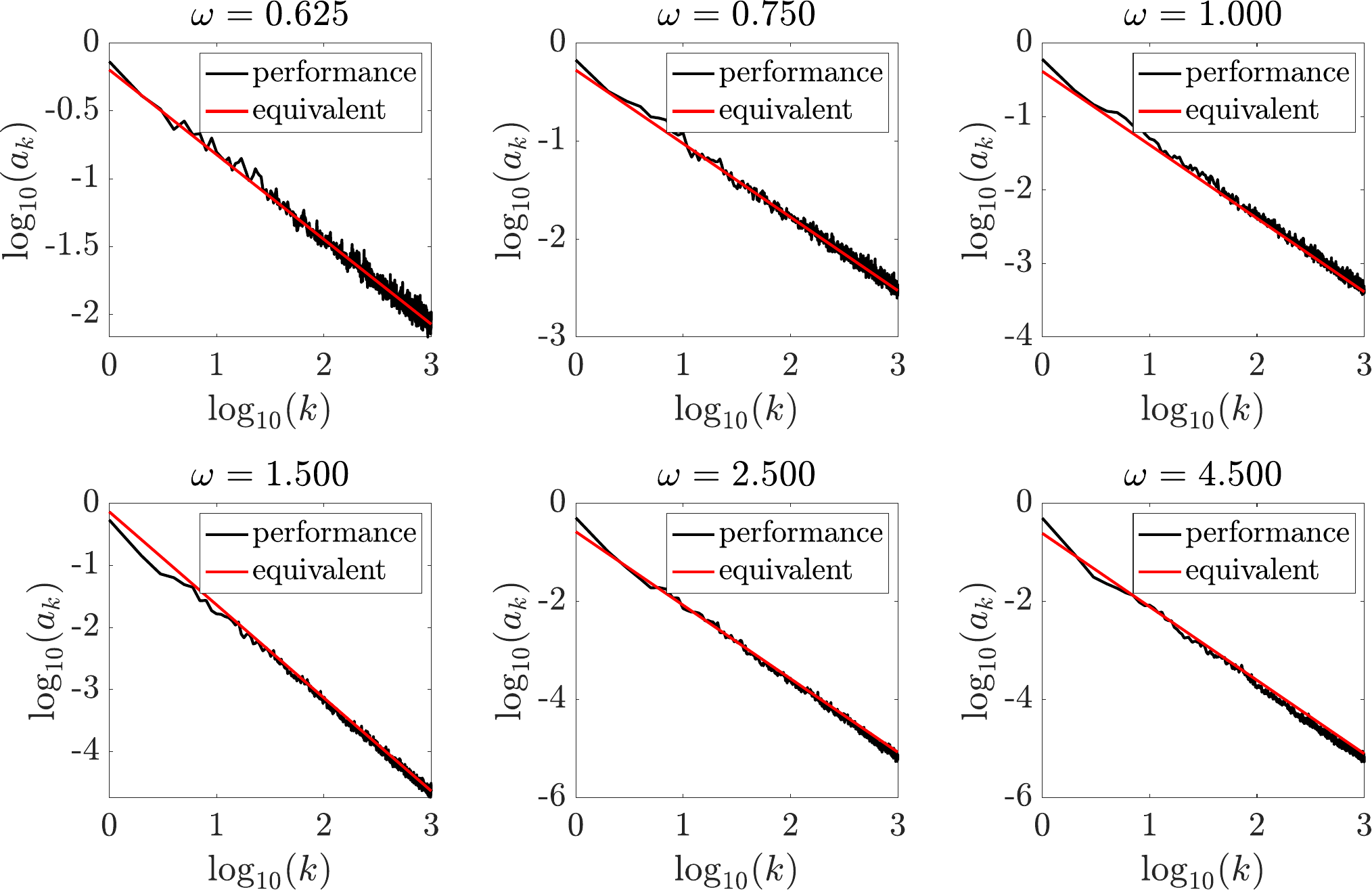}
\end{center}

\vspace*{-.75cm}

\caption{Least-mean-square algorithm for various spectral dimensions $\omega$ (with $\alpha=2$ and $\varsigma=0$ in all experiments, averaged over 20 replications), for $\omega<2 - \frac{1}{\alpha}$, where the convergence rate is proportional to $1/k^{\omega}$ and for $\omega>2 - \frac{1}{\alpha}$, where the convergence rate is proportional to  
$1/k^{2 - \frac{1}{\alpha}}$.
\label{fig:lms}}
\end{figure}

 \section{Conclusion}
 In this paper, we have shown how the $z$-transform can be used to derive asymptotic equivalents of sequences that naturally appear in the optimization of quadratic forms, with classical tools such as acceleration, averaging, or various notions of stochastic gradient descent. In all cases, the $z$-transform method allows simple computations of asymptotic equivalents (e.g., scaling laws). Several extensions are worth considering:
 \BIT
 \item We could consider asymptotic expansions~\citep{erdelyi1956asymptotic} beyond the first term, as is commonly done for Laplace approximations of integrals~\citep{bleistein1986asymptotic}. Obtaining non-asymptotic bounds using similar tools could also be interesting, in particular using complex analysis tools from~\citet{flajolet1990singularity}.
 
 \item We could extend the results in this paper to study Richardson extrapolation~\citep[see][and references therein]{bach2021effectiveness} and potentially other acceleration mechanisms \citep{brezinski2013extrapolation}. For example, when the sequence $(a_k)_{k \geqslant 0}$ satisfies $a_k = \ell + \frac{c}{k^{\alpha}} +  O( \frac{1}{k^{\alpha+1}})$, then with $b_k = \frac{2^\alpha a_k - a_{\lfloor k/2\rfloor}}{2^\alpha-1}$, we have $b_k = \ell + O(\frac{1}{k^{\alpha+1}})$ and $B(z) = \frac{2^\alpha A(z)  - (1+z) A(z^2)}{2^\alpha-1}$, and asymptotic expansions can thus be obtained.
 \item The $z$-transform method is particularly adapted to linear iterations. It would be interesting to consider extensions to nonlinear recursions for convex optimization (e.g., optimization of convex functions) using tools to obtain tight worst-case bounds \citep{lessard2016analysis,taylor2017exact}.

\item The $z$-transform method is adapted to linear recursions with time-varying coefficients that are rational functions of the sequence index, as it leads to ordinary differential equations. Extensions to more general time-varying coefficients, e.g.,  proportional to $1/\sqrt{k}$, or more generally $1/k^\alpha$ for $\alpha$ not an integer, would lead to more general results, in particular in the context of stochastic gradient descent~\citep{moulines2011non}, {but are not directly obtainable using the $z$-transform technique.}

\item Scaling laws that are not of the form $1/k^\omega$ for some $\omega>0$ have been derived by~\citet{kunstner2025scaling} in contexts where $\alpha \leqslant 1$ in Lemma~\ref{lemma:GD}. The $z$-transform could also lead there to simpler proofs.

\item While momentum techniques (such as heavy-ball or Nesterov acceleration) are often used in practice, their convergence properties for stochastic gradient descent remain an active area of research. We could consider Nesterov single-pass SGD from \mysec{nesterov-1} with the same analysis as \mysec{sgd}. We could also consider other recursions, such as \citet{jain2018accelerating}, \citet{varre2022accelerated} or~\citet{even2021continuized}, along the lines of what has been proposed by \citet{yarotsky2024sgd,ferbach2025dimension}. The analysis of~\citet{aybat2019universally} could also be strengthened. Recently introduced random stepsizes that lead to acceleration \citep{altschuler2024acceleration} could also be studied within our tools to obtain scaling laws.

 \item  Multiple-pass stochastic gradient descent~\citep{pillaud2018statistical} on the model from \eq{modelLMS} with a finite dataset could be considered, potentially adding features within (S)GD to get a model close to the one of \citet{paquette20244} but with a simpler analysis. Plain gradient descent could also be studied to obtain an asymptotic version of the near-equivalence between computational regularization and $\ell_2$-regularization in least-squares regression~\citep{ali2019continuous}.

\item Primal-dual algorithms (e.g., extragradient vs. optimistic gradient ascent/descent or proximal point method), as analyzed by~\citet{mokhtari2020unified}, could also be considered to compare their relative scaling behavior.
\item Variance reduction techniques such as SAG or SAGA \citep[see][and references therein]{gower2020variance}, could also be considered using similar tools, in particular in the non-strongly convex regime where convergence is sublinear.

\item Beyond optimization, scaling laws for sampling algorithms can naturally be obtained with Gaussian distributions, for Langevin (unadjusted Langevin algorithm) and its accelerated versions (underdamped Langevin dynamics)~\citep[see][and references therein]{chewi2023optimization}. Note here that we need to make assumptions about the decay of the eigenvalues of the precision matrix (the inverse of the covariance matrix). This could also be done for Gibbs sampling, or for more general Markov chains. 
\item Algorithm design: Can we use the $z$-transform interpretation to design algorithms with favorable convergence guarantees?
\EIT

   \subsection*{Acknowledgements}
   The author thanks Adrien Taylor, Ayoub Melliti, Nicolas Flammarion, Baptiste Goujaud, and Rapha\"el Berthier for insightful discussions related to this work. Comments from the reviewers were greatly appreciated.
   This work has received support from the French government, managed by the National Research Agency,
under the France 2030 program with the reference ``PR[AI]RIE-PSAI'' (ANR-23-IACL-0008).
  
  \appendix

\section{Convolutions of rational functions}
\label{app:conv}

In this section, we compute convolutions $A \ast B(z)$ for rational functions $A$ and $B$, starting from $$
\frac{1}{a- uz} * \frac{1}{b- vz} = \frac{1}{ab- uvz},
$$
which is a direct consequence of the $z$-transform of the sequence $( a^{-1} ( u a^{-1})^k)_{k \geqslant 0}$ being $z \mapsto \frac{1}{a- uz}$.

\paragraph{Unique multiple roots.} Since convolutions are bilinear operations, we can differentiate with respect to $a$ $k$ times and to $b$ $\ell$ times to get, using the Leibniz formula:\footnote{See \url{https://en.wikipedia.org/wiki/General_Leibniz_rule}.}
\BEAS
& & \frac{(-1)^k k! }{(a- uz)^{k+1}} * \frac{(-1)^\ell \ell! }{(b- vz)^{\ell+1}}  \\
& \!\!\!= \!\!\!& \frac{d^{k+\ell}}{d a^k db^\ell} \Big( \frac{1}{ab- uvz} \Big)
=  \frac{d^{k}}{d a^k} \Big( \frac{(-1)^\ell \ell! a^\ell}{(ab- uvz)^{\ell+1}} \Big)
\\
&\!\!\! = \!\!\!&(-1)^\ell \ell!  \sum_{i=0}^{\min\{k,\ell\}} \frac{k!}{i! (k-i)!} \Big[  \frac{\ell!}{(\ell-i)!} a^{\ell-i} \Big]
\cdot \Big[ \frac{(k+\ell-i)! b^{k-i}}{\ell!} \frac{(-1)^{k-i} }{(ab- uvz)^{\ell+k+1-i}} \Big],
\EEAS
leading to
$\displaystyle
\frac{1}{(a- uz)^{k+1}} * \frac{1 }{(b- vz)^{\ell+1}} 
= \sum_{i=0}^{\min\{k,\ell\}} (-1)^{i}\frac{(k+\ell-i)! }{i! (k-i)! (\ell-i)!}   \frac{  a^{\ell-i}   b^{k-i} }{(ab- uvz)^{\ell+k+1-i}}.
$

\paragraph{Quadratic denominators.}
We will also need convolutions of rational functions with quadratic denominators, that is, by using complex conjugate roots and bilinearity,
\BEA
 \notag& &  \frac{1}{ (1-z)^2 + \lambda }
* \frac{1}{ (1-z)^2 + \mu }  \\
\notag & \!\!\!= \!\!\!& \frac{-1}{4 \sqrt{\lambda\mu}}\Big(
\frac{1}{1-z -i \sqrt{\lambda}}
-\frac{1}{1-z +i \sqrt{\lambda}}
\Big) \ast \Big(
\frac{1}{1-z -i \sqrt{\mu}}
-\frac{1}{1-z +i \sqrt{\mu}}
\Big) \\
\notag&\!\!\! = \!\!\!&\frac{1}{4 \sqrt{\lambda\mu}}\Big(
\frac{1}{(1-i \sqrt{\lambda})(1+i \sqrt{\mu}) - z}
+
\frac{1}{(1+i \sqrt{\lambda})(1-i \sqrt{\mu}) - z}
\Big) 
\\
\notag& & -  \frac{1}{4 \sqrt{\lambda\mu}}\Big(
\frac{1}{(1-i \sqrt{\lambda})(1-i \sqrt{\mu}) - z}
+
\frac{1}{(1+i \sqrt{\lambda})(1+i \sqrt{\mu}) - z}
\Big) \\
\notag& \!\!\!= \!\!\!&  \frac{1}{2\sqrt{\lambda \mu}} \bigg(
\frac{1+ \sqrt{\lambda\mu} - z}{( 1+ \sqrt{\lambda \mu} - z)^2 + ( \sqrt{\lambda} - \sqrt{\mu})^2}
-\frac{1- \sqrt{\lambda\mu} - z}{( 1- \sqrt{\lambda \mu} - z)^2 + ( \sqrt{\lambda} + \sqrt{\mu})^2}
\bigg)\\
\label{eq:quadden}& \!\!\!= \!\!\!& 
\frac{ (\lambda+1)(\mu+1) - z^2}{( ( 1-z)^2 + \lambda \mu +\lambda + \mu)^2 - 4\lambda \mu z^2}.
\EEA

  \section{Proof of Proposition~\ref{prop:nesterov-1}}
  \label{app:nesterov-1}
We use convolution formulas obtained from \eq{quadden} in Appendix~\ref{app:conv}, with $B(z,\lambda)
=  \frac{1}{(1-\lambda) (1-z)^2 + \lambda }
$ defined in \eq{Bnest}. We obtain several expressions depending on the choice of partial-fraction decomposition for the convolution $A(z) = B(z,\lambda) 
* B(z,\lambda)$:\footnote{See Mathematica notebook {\url{https://www.di.ens.fr/~fbach/ztf/nesterov.nb}}.}

\BEA
\notag \!\!\!\!\!\!\!A(z) &\!\!\! =\!\!\!&   \frac{1}{ (1-\lambda) (1-z)^2 + \lambda }
* \frac{1}{ (1-\lambda) (1-z)^2 + \lambda } \\
\notag &\!\!\! = \!\!\!& \frac{-\lambda  z+z+1}{((\lambda -1) z+1) \left((\lambda -1)^2 z^2+\left(-4 \lambda ^2+6 \lambda -2\right) z+1\right)} \mbox{ (no expansion)}\\
  \label{eq:nestzroot}&\!\!\! = \!\!\!& \frac{1}{2 (-1+\lambda ) \lambda  \big(-1+z-\frac{\lambda }{1-\lambda }\big)}+\frac{1}{4 (1-\lambda ) \lambda  \big(-1+z+\frac{\lambda
}{1-\lambda }-2 i \sqrt{\frac{\lambda }{1-\lambda }}\big)} \\
\notag & & \hspace*{3cm} +\frac{1}{4 (1-\lambda ) \lambda  \big(-1+z+\frac{\lambda }{1-\lambda }+2 i \sqrt{\frac{\lambda
}{1-\lambda }}\big)}  \mbox{ (expansion in $z$)}. 
\\
 \label{eq:nestlambdaroot}&\!\!\! = \!\!\!&\frac{z}{2 \left(\sqrt{z}-1\right) \left(\sqrt{z}+1\right) (\lambda  z-z+1)}+\frac{-\sqrt{z}-2}{4 \left(\sqrt{z}+1\right) \left(\lambda  \left(z+2 \sqrt{z}\right)-z-2 \sqrt{z}-1\right)} \\
 \notag & & \hspace*{3cm} +\frac{2-\sqrt{z}}{4 \left(\sqrt{z}-1\right) \left(\lambda  \left(z-2 \sqrt{z}\right)-z+2 \sqrt{z}-1\right)}
 \mbox{ (expansion in $\lambda$)}. 
\EEA

Starting with \eq{nestlambdaroot} (partial-fraction expansion in $\lambda$), we obtain directly an expression in terms of the \sti transform, as
\BEA
\label{eq:Alo} A(z) & =  &\frac{S\left(\frac{1}{z}-1\right)}{2 (z-1)}-\frac{S\Big(-\frac{\left(\sqrt{z}-1\right)^2}{z-2 \sqrt{z}}\Big)}{4 \left(z-\sqrt{z}\right)}-\frac{S\Big(-\frac{\left(\sqrt{z}+1\right)^2}{z+2 \sqrt{z}}\Big)}{4 \left(z+\sqrt{z}\right)},
\EEA
with the following equivalent when $z \to 1$:
\BEA
\label{eq:Aloc} A(z)
&\sim & \frac{S(1-z)}{2 (z-1)}-\frac{S\left(\frac{1}{4} (z-1)^2\right)}{2 (z-1)}-\frac{1}{8} S\Big(-\frac{4}{3}\Big).
\EEA
\eq{Aloc} will lead to expansions of the $z$-transform $A(z)$ and its derivatives, thus characterizing the asymptotic equivalent if it exists. {If the spectral measure has a density around~$0$, then Tauberian conditions from \mysec{z} could be used. However, we want to apply to a discrete measure, and the direct asymptotic proof below shows that such equivalents exist. We then consider the $z$-transform  to obtain the same equivalents directly.}

\paragraph{{Direct equivalents}.}  {For the power-law spectral model assumed in Proposition~\ref{prop:nesterov-1}, let  $\phi_i=\arcsin\!\left(\sqrt{\gamma L}\,i^{-\alpha/2}\right)$. The coefficient of $z^k$ in $B(z,\gamma L i^{-\alpha})$ is
$$
(1-\gamma L i^{-\alpha})^{k/2}
\frac{\sin((k+1)\phi_i)}{\sqrt{\gamma L}\,i^{-\alpha/2}},
$$
and the definition of the spectral measure gives the exact identity
$$
a_k=\frac{\Delta^2}{2\gamma}\sum_{i=1}^{+\infty}i^{-\beta}(1-\gamma L i^{-\alpha})^k\sin^2((k+1)\phi_i).
$$
}

{(a) If $0<\omega<1$, then $\beta<1$. With the mapping $i\sim k^{2/\alpha}x$, dominated Riemann-sum convergence gives, with $c$ defined in Lemma~\ref{lemma:GD},
$$
k^{2\omega-2}a_k\longrightarrow
\frac{\Delta^2}{2\gamma}\int_0^{+\infty}x^{-\beta}
\sin^2\!\left(\sqrt{\gamma L}\,x^{-\alpha/2}\right)dx
= c\,2^{1-2\omega}
\frac{\Gamma(\omega)\Gamma(1-\omega)}{\Gamma(3-2\omega)},
$$
where the last equality is obtained by the change of variable $t=\sqrt{\gamma L}\,x^{-\alpha/2}$ and the identity
$$
\int_0^{+\infty}t^{2\omega-3}\sin^2(t)dt
=2^{-2\omega}\frac{\Gamma(\omega)\Gamma(1-\omega)}{\Gamma(3-2\omega)}.
$$}

{(b) If $\omega>1$, then $\beta>1$. Using $\sin^2(t)=(1-\cos(2t))/2$, the non-oscillatory part satisfies
$$
\frac{\Delta^2}{4\gamma}\sum_{i=1}^{+\infty}i^{-\beta}(1-\gamma L i^{-\alpha})^k
\sim \frac{c}{2}\Gamma(\omega-1)k^{1-\omega},
$$
while the oscillatory part is $o(k^{1-\omega})$, leading to 
$
a_k\sim \frac{c}{2}\Gamma(\omega-1)k^{1-\omega}
$. }

{To justify the estimate for the oscillatory part, we use the van der Corput
method \citep{graham-kolesnik-1991}. Set  
\(n=k^{1/\alpha}\), and (with $\mathrm{i}^2=-1$)
\[
f_k(x)=\frac{k+1}{\pi}
\arcsin\bigl(\sqrt{\gamma L}\,x^{-\alpha/2}\bigr),
\qquad e(t)=\exp(2\pi\mathrm{i}t),
\]
so that \(e(f_k(i))=\exp(2\mathrm{i}(k+1)\phi_i)\). On a dyadic
interval \(N<i\leqslant 2N\), set \(F_N=kN^{-\alpha/2}\). For every fixed
integer \(r\geqslant1\), differentiation of the preceding expression gives that
$
  |f_k^{(r)}(x)|$ is of order $ F_NN^{-r}$ when 
$ N\leqslant x\leqslant 2N$,
uniformly as \(N\) tends to infinity. Choose an integer \(\ell\geqslant0\)
such that \(\ell+2>\alpha/2\). The higher-derivative van der Corput
estimate \citep[Theorem~2.9]{graham-kolesnik-1991} then gives, uniformly
for \(N\leqslant y\leqslant 2N\),
\[
 \bigg|\sum_{N<i\leqslant y}e(f_k(i))\bigg|
 =o \bigg( N\Big[
 \Big(\frac{F_N}{N^{\ell+2}}\Big)^{1/(4\cdot2^\ell-2)}
 +F_N^{-1}\Big]\bigg).
\]
For \(\varepsilon n\leqslant N\leqslant Mn\), the right-hand side is \(o(N)\),
uniformly in \(N\). Since
\(x^{-\beta}(1-\gamma L x^{-\alpha})^k\) is unimodal, partial summation on the
finitely many dyadic intervals covering \([\varepsilon n,Mn]\) shows
that the oscillatory sum over this range is \(o(n^{1-\beta})\). The two
remaining ranges satisfy
\[
 n^{\beta-1}\sum_{i\leqslant\varepsilon n}
 i^{-\beta}(1-\gamma L i^{-\alpha})^k
 \leqslant
 \frac1n\sum_{i\leqslant\varepsilon n}
 (i/n)^{-\beta}e^{-\gamma L (i/n)^{-\alpha}}
 \longrightarrow
 \int_0^\varepsilon x^{-\beta}e^{-\gamma L x^{-\alpha}}\,dx,
\]
and
$
 n^{\beta-1}\sum_{i>Mn}i^{-\beta}=o( M^{1-\beta}).
$
Letting first \(k\) tend to infinity, then \(\varepsilon\) tend to zero
and \(M\) tend to infinity, proves that the oscillatory part is
\(o(n^{1-\beta})=o(k^{1-\omega})\).}

{(c) When \(\omega=1\), we have \(\beta=1\) and again write
\(n=k^{1/\alpha}\). The ranges \(i\leqslant n\) and \(i\geqslant n^2\) contribute
only \(O(1)\): for the first range, use
\((1-\gamma L i^{-\alpha})^k\leqslant \exp(-\gamma L (n/i)^\alpha)\), and for the second use
\(\phi_i=o( i^{-\alpha/2})\), whence
\(\sin^2((k+1)\phi_i)=o( k^2i^{-\alpha})\). Moreover,
\[
 \sum_{n<i<n^2}\frac{(1-\gamma L i^{-\alpha})^k}{i}
 =\log n+O(1),
\]
since \(1-(1-u)^k\leqslant ku\). It remains to show that the corresponding
cosine sum is \(o(\log n)\). Fix \(\delta\in(0,1)\). On every dyadic
block \(N<i\leqslant 2N\) contained in \([n,n^{2-\delta}]\), the preceding
van der Corput estimate applies, with
$
 F_N\geqslant n^{\alpha\delta/2}$,
and $
 \frac{F_N}{N^{\ell+2}}
 \leqslant n^{\alpha/2-(\ell+2)}.
$
Partial summation therefore shows that the contribution of each such
block is \(o(1)\), uniformly in \(N\); summing over the \(O(\log n)\)
blocks gives \(o(\log n)\). The remaining interval
\([n^{2-\delta},n^2]\) has absolute contribution at most
\(\delta\log n+O(1)\). Letting \(\delta\) tend to zero proves that the
full cosine sum is \(o(\log n)\). Consequently,
$
 a_k
 =\frac{\Delta^2}{2\gamma}
   \left(\frac12\log n+o(\log n)\right)
 =\frac{\Delta^2}{4\alpha\gamma}\log k+o(\log k)
 \sim\frac{c}{2}\log k,
$
because \(c=\Delta^2/(2\alpha\gamma)\) when \(\omega=1\).}

{
Dividing these three  equivalents by $k^2$ proves Proposition~\ref{prop:nesterov-1}. We now recover the corresponding singular expansions of $A(z)$.}

\paragraph{{Use of $z$-transform.}}
When $\omega \in (0,1)$, we have the equivalent for the \sti transform from \eq{sti}
$$
S(u) \sim    c \Gamma(1-\omega) \Gamma(\omega) u^{\omega -1},
$$
and thus
we get the equivalent from \eq{Aloc}
$$A(z) \sim
 c \Gamma(1-\omega) \Gamma(\omega)  \Big[
 - \frac{1}{2} (1-z)^{\omega-2} 
 + \frac{1}{2} \frac{1}{4^{\omega-1}} (1-z)^{2\omega - 3}
 \Big],
$$
which is equivalent to (since the second term dominates)
\BEQ
\label{eq:proof-nest1}
A(z) \sim
 c \Gamma(1-\omega) \Gamma(\omega)   
 \frac{1}{2} \frac{1}{4^{\omega-1}} (1-z)^{2\omega - 3}.
\EEQ
{This singular expansion is consistent with the coefficient equivalent established directly above.}
   
When $\omega \in (1,2)$, we have from \eq{sti} an equivalent for the derivative of the \sti transform:
 $$
S'(u) \sim - c  \Gamma(2-\omega) \Gamma(\omega)  u^{\omega - 2 },
$$
leading by integration to
$$
S(u) = S(0) - c  \Gamma(2-\omega) \Gamma(\omega) \frac{u^{\omega - 1 }}{\omega - 1} + o(u^{\omega - 1 }) 
.$$
Thus, only the first two terms in \eq{Aloc} dominate, with the value $S(0)$ cancelling out, leading to
\BEQ
\label{eq:proof-nest2}
A(z) \sim
 -c \frac{\Gamma(2-\omega) \Gamma(\omega) }{\omega-1} \Big[
 - \frac{1}{2} (1-z)^{\omega-2} 
 + \frac{1}{2} \frac{1}{4^{\omega-1}} (1-z)^{2\omega - 3}
 \Big] \sim  c \frac{\Gamma(2-\omega) \Gamma(\omega) }{\omega-1}  
  \frac{1}{2} (1-z)^{\omega-2} .
\EEQ
{This singular expansion is consistent with the coefficient equivalent established directly above.}

When $\omega >  2$, we simply have to obtain the same equivalent as above, and take derivatives with respect to $z$ of \eq{Alo}. It can be easily checked that the dominant term still corresponds to the term
$\frac{S\left(\frac{1}{z}-1\right)}{2 (z-1)}${; this is consistent with the coefficient equivalent established directly above.}

Finally, for  $\omega = 1$, we have from \eq{sti} $S'(u) \sim - c u^{-1}$, leading to by integration $S(u) \sim c \log \frac{1}{u}$ around $u=0$. The first two terms of \eq{Aloc} then both contribute and lead to $A(z) \sim \frac{c}{2(1-z)}\log \frac{1}{1-z}$, {in agreement with the direct coefficient equivalent established above.}

\section{Proof of Lemma~\ref{lemma:nest-2}}

\label{app:lemma:nest-2}
  Multiplying by $z^{2\rho-2}$, {\eq{nesterov-22}} is equivalent to
$$ z^{2\rho-2} \big[ 1 - (2z -z^2)  ( 1 - \lambda) \big]  B'(z,\lambda)  + 2 (\rho-1) z^{2\rho-3} [ B(z,\lambda)-1]
\cdot  [ 1   - z ( 1-\lambda) ]
=    ( 2\rho-1) z^{2\rho-2}.
$$
Denoting $D(z,\lambda) = z^{2\rho-2} ( B(z,\lambda) - 1 ) $, and $c(z) =  1 - (2z -z^2)  ( 1 - \lambda) $, we get
$c'(z) = -2(1-z) (1-\lambda)$ and:
$$ c(z) D'(z,\lambda) - (\rho-1) c'(z) D(z,\lambda) =    ( 2\rho-1) z^{2\rho-2}.
$$
Taking $2\rho-1$ derivatives, and using Leibniz formula, we get:
\BEAS
   c(z) D^{(2\rho)}(z,\lambda) + (2\rho-1 ) c'(z)D^{(2\rho-1)}(z,\lambda) 
+  (2\rho-1)(\rho-1) c''(z)D^{(2\rho-2)}(z,\lambda) 
\\
 - (\rho-1) c'(z) D^{(2\rho-1)}(z,\lambda)  - (\rho-1) (2\rho -1) c''(z) D^{(2\rho-2)}(z,\lambda) & = & 0.
\EEAS
Thus
$$   c(z) D^{(2\rho)}(z,\lambda) + \rho c'(z)D^{(2\rho-1)}(z,\lambda) 
=0,
$$
leading to $\ds  c(z)^\rho D^{(2\rho-1)}(z,\lambda)  = \alpha(\lambda),$
for some function $\alpha$. Looking at $z=0$, we have $c(0) = 1$, and
$D^{(2\rho-1)}(0,\lambda)=(2\rho-{1})!$, leading to
the desired result.

  \section{Link with spectral dimension condition from~\citet{berthier2020accelerated}}
 \label{app:sdb}
 In this section, we provide a short proof, that Assumption  \textbf{(A1)} implies that $\sigma((0,u)) \sim_{u \to 0} \frac{c}{\omega}  u^\omega$ (which is the condition from~\citet{berthier2020accelerated}).
 
{Let $m$ be an integer such that $m>\omega$. By Assumption \textbf{(A1)}, we have 
$$
\int_0^1
\frac{d\sigma(\lambda)}{(u+\lambda)^m} = 
\frac{(-1)^{m-1}}{(m-1)!}S^{(m-1)}(u) \sim
c\,\frac{\Gamma(\omega)\Gamma(m-\omega)}
        {\Gamma(m)}
u^{\omega-m},
$$
which matches the expression for the measure $d \left[\lambda^\omega\right] = \frac{c}{\omega} \omega \lambda^{\omega-1}d\lambda$ instead of $d\sigma$, as we have $
\frac{c}{\omega}\int_0^{+\infty}
\frac{ \omega \lambda^{\omega-1}d\lambda}
     {(u+\lambda)^m}
=
c\,\frac{\Gamma(\omega)\Gamma(m-\omega)}
        {\Gamma(m)}
u^{\omega-m}.
$
Following the reciprocal-variable reduction of \citet[proof of Theorem A.7]{LangerWoracek2025}, we apply the corresponding generalized Stieltjes-transform form of Karamata’s theorem \citep[Theorem~1.7.4]{BinghamGoldieTeugels1987}, to get, when $u$ tends to $0+$,
 $
\sigma((0,u)) 
\sim
\frac{c}{\omega}u^\omega.
$}


\section{Proof of Proposition~\ref{prop:AVGD}}
\label{app:AVGD}

We have, following \eq{AGDP}
\BEAS
& & \frac{1}{2} \sum_{i \geqslant 1} h_i
\frac{ \langle \Theta(z), u_i \rangle}{1-z} * \frac{ \langle \Theta(z), u_i \rangle}{1-z} \\
& = & 
\frac{1}{2\gamma} \sum_{i \geqslant 1} \gamma h_i
\langle \theta_0, u_i \rangle^2 \big(\frac{1 }{1-z} \frac{1}{1 - z( 1- \gamma h_i)} \big)
* \big(\frac{1 }{1-z} \frac{1}{1 - z( 1- \gamma h_i)} \big)
\\
& = & 
\frac{1}{2\gamma} \int_0^1 
 \big(\frac{1 }{1-z} \frac{1}{1 - z( 1- \lambda)} \big)
* \big(\frac{1 }{1-z} \frac{1}{1 - z( 1- \lambda)} \big)
d\sigma(\lambda)
\\
& = & 
  \int_0^1 
\frac{1}{\lambda^2}\Big(
\frac{1}{1-z}  + \frac{(1-\lambda)^2}{1-z(1-\lambda)^2} - \frac{ 2 (1-\lambda)}{1-z(1-\lambda)}
\Big)
d\sigma(\lambda) \\
& = & 
  \int_0^1  \Big(-\frac{2 z}{(z-1)^2 (\lambda  z-z+1)}-\frac{1}{2 \left(\sqrt{z}+1\right)^2 \left(\lambda  \sqrt{z}-\sqrt{z}-1\right)} \\
  & & \hspace*{6.5cm} +\frac{1}{2 \left(\sqrt{z}-1\right)^2 \left(\lambda  \sqrt{z}-\sqrt{z}+1\right)}\Big)
d\sigma(\lambda) \\
& = & -\frac{2 S\left( \frac{1}{z}-1\right)}{(z-1)^2}+\frac{S\left(\frac{1}{\sqrt{z}}-1\right)}{2 \left(\sqrt{z}-1\right)^2 \sqrt{z}}-\frac{S\left(-\frac{1}{\sqrt{z}}-1\right)}{2 \left(\sqrt{z}+1\right)^2 \sqrt{z}}.
 \EEAS
 Thus, when $z$ tends to one, we get the equivalent
$$
 -\frac{2 S\left(1-z\right)}{(1-z)^2}+\frac{2 S\left(\frac{1-z}{2}\right)}{(1-z)^2}-\frac{S\left(-2\right)}{8}$$
When $\omega>2$, we see that we obtain an equivalent of $-\frac{S'({0})}{1-z}$, which leads to an overall rate of $\propto \frac{1}{k^2}$, which is compatible with existing results~\citep{defossez2015averaged} with a rate proportional to $\frac{1}{2k^2} \langle \theta_0  , H^{-1}  \theta_0  \rangle$.

When $\omega \in (0,1) $, we get the equivalent 
$$
c {\Gamma(1-\omega) \Gamma(\omega)}  
\big[
 2^{2-\omega} - 2
\big] ( 1- z)^{\omega -3} = 2 c {  \Gamma(\omega)^2}  
\frac{
 2^{1-\omega} -1
}{1-\omega} ( 1- z)^{\omega -3}  ,
$$
which is valid for all $\omega \in (0,2)$ (with the limit $\log 2$ instead of $\frac{
 2^{1-\omega} -1
}{1-\omega} $ for $\omega=1$). This leads to the equivalent for averaged gradient descent
$
2 c \frac{  \Gamma(\omega)^2} {\Gamma(3-\omega)}
\frac{
 2^{1-\omega} -1
}{1-\omega} \frac{1}{k^\omega}.$

\section{Variance estimates for SGD}
\label{app:var}
In this appendix, we outline how to obtain an asymptotic equivalent for the variance {term} of the performance $\frac{1}{2} \langle \Theta_k, H \rangle$ using notations from \mysec{sgd} (in finite dimension for simplicity). We consider the simplest case where $\varsigma=0$ and the normalized random variables are Rademacher random variables. Then, a short calculation shows that
$$
\E \Big[ {\rm vec} \big( (\theta_k \circ \theta_k)  (\theta_k \circ \theta_k)^\top  \big) \Big]
= U^k {\rm vec} \big( (\theta_0 \circ \theta_0) (\theta_0 \circ \theta_0)^\top  \big),
$$
where $U$ is an operator defined as
\BEAS
U & = &  
\idm \otimes \idm - 2 \gamma \big[  \Diag(h) \otimes \idm +\idm \otimes \Diag(h) \big] 
+ 8 \gamma^2 \Diag(h) \otimes \Diag(h)
\\
&&
+   \gamma^2  \Big[  \idm \otimes hh^\top  +    hh^\top \otimes \idm  \Big]  
 -6 \gamma^3 \Big[  \Diag(h) \otimes hh^\top  + hh^\top  \otimes \Diag(h) \Big] 
 + 3 \gamma^4 hh^\top \otimes hh^\top.
 \EEAS
In order to obtain the variance {term} of the performance, the criterion we look at is 
\BEAS
b_k & = & \E \Big[  \big(( \theta_k \circ \theta_k)^\top h \big)^2\Big]
=
 h^\top  \E \big[ ( \theta_k \circ \theta_k) ( \theta_k \circ \theta_k)^\top \big] h 
\\
&  = &  {\rm vec}(hh^\top)^\top 
 {\rm vec} \big( (\theta_k \circ \theta_k)  (\theta_k \circ \theta_k)^\top  \big)
=  {\rm vec}(hh^\top)^\top 
 U^k {\rm vec} \big( (\theta_0 \circ \theta_0) (\theta_0 \circ \theta_0)^\top  \big).
\EEAS
This leads to the $z$-transform
\BEAS
B(z) & = &
 {\rm vec}(hh^\top)^\top 
 ( \idm - z U)^{-1} {\rm vec} \big( (\theta_0 \circ \theta_0) (\theta_0 \circ \theta_0)^\top  \big).
\EEAS

 If we take the equivalent when $\gamma \to 0$, then we only have
 $$ B(z) = \sum_{i,j=1}^d \delta_i^2 \delta_j^2 \frac{1}{1-z + 2z \gamma (h_i+h_j)}
 = \int_0^{+\infty} \sum_{i,j=1}^d \delta_i^2 \delta_j^2 \exp(-y( 1-z + 2z \gamma (h_i+h_j))) dy$$
 for which an equivalent can be found to be exactly the square of the expectation.

\bibliography{ztransform}

 \end{document}